%% file: m2g_pca_arXiv.tex
\documentclass[a4paper,11pt]{article}

\usepackage{amsmath,amssymb,amsthm,bbm,bm}  
\usepackage{authblk}  
\usepackage[british]{babel}
\usepackage{balance}
\usepackage[utf8]{inputenc}
\usepackage[backref=true,sorting=none,doi=false,isbn=false,giveninits=true,style=nature,url=false]{biblatex} 
\usepackage{booktabs} 
\usepackage{caption}
\usepackage{color}
\usepackage{comment}
\usepackage{dsfont}
\usepackage{enumitem}
\usepackage[T1]{fontenc}
\usepackage[top=2.5cm, bottom=2.5cm, left=2.5cm, right=2.5cm]{geometry}
\usepackage{graphicx}
\usepackage{graphbox}       
\PassOptionsToPackage{hyphens}{url}  
\usepackage[colorlinks = true,
            linkcolor = OliveGreen,
            urlcolor  = blue,
            citecolor = MidnightBlue,
            anchorcolor = blue]{hyperref}
\usepackage{cleveref}  
\usepackage{csquotes}  
\usepackage[mono=false]{libertine}  
\usepackage{mathtools}
\usepackage{microtype}      
\usepackage{multirow}    
\usepackage{nicefrac}
\usepackage{subfig} 
\usepackage{thmtools}
\usepackage{wrapfig}
\usepackage[dvipsnames]{xcolor}

\usepackage{commands}

\allowdisplaybreaks

\newcommand{\citet}[1]{\textcite{#1}}
\newcommand{\citep}[1]{\cite{#1}}

\title{Memorisation, convergence and generalisation\\ in generative models}

\date{\today}

\author[1]{Antoine Maillard}
\author[2]{Sebastian Goldt\thanks{sgoldt@sissa.it}}
\affil[1]{INRIA Paris \& DI ENS, PSL University, Paris, France}
\affil[2]{International School of Advanced Studies (SISSA), Trieste, Italy}

\begin{document}

\maketitle

\begin{abstract}
    \noindent \input{m2g_pca_abstract.tex}
\end{abstract}

\section{Introduction}

\input{m2g_pca_main.tex}

\section*{Acknowledgements}

\input{m2g_pca_acknowledgements}

\newpage
\printbibliography

\newpage
\appendix

\onecolumn
\numberwithin{equation}{section}
\numberwithin{figure}{section}

\begin{center}
  \LARGE SUPPLEMENTAL MATERIAL
\end{center}

\input{m2g_pca_supp.tex}

\end{document}

%% file: m2g_pca_abstract.tex
Generative neural networks learn how to produce highly realistic images from a
large, but finite number of examples -- or do they simply memorise their
training set?
To settle this question, Kadkhodaie, Guth, Simoncelli and Mallat
(ICLR ’24)~\citep{kadkhodaie2024generalization} trained diffusion models
independently on disjoint subsets of a dataset and showed that they converge to
nearly the same density when the number of training images is large enough.
This result raises two basic questions: how much data do you need for convergence, and
what does convergence capture about learning the data distribution? 
Here, we address these questions by providing an exact analytical
characterisation of the transition from memorisation to generalisation in linear
generative models.
We find that these models memorise at small load, while convergence emerges
continuously when the number of samples is linear in the input dimension.
Strikingly, we find that convergence is insensitive to recovery of the principal
latent factors of the data, which are recovered in a sharp transition.
After extending our approach to data with power-law spectra, we find the same
distinction between convergence and latent recovery in our experiments with
convolutional denoisers and in the data of
\citet{kadkhodaie2024generalization}. 
We thus show that generalisation in generative models decomposes into at least
two distinct objectives: matching the bulk of the data distribution and
recovering the principal latent factors. These objectives correspond to two different
distances between true and learnt data distribution, and only the first one is
captured by convergence.

%% file: m2g_pca_main.tex
The success of modern \emph{generative} neural networks poses a new
challenge for the theory of generalisation in neural networks. The first challenge
arose in supervised learning,
where one aims to predict a property of an input, when it was observed
that deep neural networks
generalise well even if they achieve (near) zero training
loss~\citep{neyshabur2014search, zhang2016understanding,
neyshabur2017exploring}. This phenomenon appeared to contradict
classical statistical learning theory~\citep{hastie2009introduction}, which
predicts that models with zero training loss overfit the training data and
generalise poorly. A large body of recent work has resolved this paradox by
analysing simple models like kernel methods and small neural networks where
\emph{benign overfitting}~\citep{muthukumar2020harmless, bartlett2020benign,
tsigler2023benign} can be understood theoretically~\citep{belkin2018overfitting,
belkin2019reconciling, dascoli2020double, mei2022generalization,
hastie2022surprises}. The key insight is that interpolating solutions can
generalise well due to implicit biases induced by the learning 
algorithm~\citep{gunasekar2018characterizing, chizat2020implicit}.

However, this perspective does not immediately carry over to generative models, where
achieving zero training loss typically leads to explicit memorisation: sampling
from such a model simply reproduces training examples (or small perturbations of
them)~\citep{carlini2021extracting,somepalli2023understanding}.
For generative models, interpolation is therefore no longer benign, but
detrimental to the goal of generating new samples. Taking a step back, it is not
even clear how to measure the ability of a generative model to generalise, i.e.\
to generate new samples of high quality, in a way that is easily
operationalised. Achieving small test error when denoising a test image, or
matching the low-order statistics of the training data, does not necessarily imply
that the model is able to generate high-quality and diverse samples from the
true data distribution.


These considerations have motivated the search for notions of generalisation
tailored to generative modelling. In a remarkable paper,
\citet{kadkhodaie2024generalization} (KGSM) related generalisation to the
variability of generative models trained on different realisations of a dataset.
They trained two denoising diffusion models independently on disjoint subsets of
CelebA, a large-scale data set of images, each containing $n$ different images.
Denoising diffusion models~\citep{sohl2015deep, ho2020denoising} are trained to reconstruct clean images
from noisy versions; they generate new samples by starting from a latent
variable that is pure noise and iteratively ``denoising'' it into an image. When
training on two large datasets, KGSM found that
sampling from the two denoisers starting from the \emph{same} latent vector
yielded almost identical high-quality images: the two denoisers had converged to
essentially the same function despite not sharing any training image. At small
sample complexity however, the models essentially memorised their training data,
returning different training images for the same latent vector (see
also~\citet{carlini2021extracting,somepalli2023understanding} for demonstrations
of memorisation). KGSM thus observed a transition from a
memorisation regime (high variability between models) to a generalisation regime
with low variability. They argued that neural networks achieve this transition
at a sample complexity that is not cursed by dimensionality by exploiting the bias of local
denoising towards smooth, geometry-aware functions.

\begin{figure*}[t!]
  \centering
  \includegraphics[trim=0 360 0 0,clip,width=\linewidth]{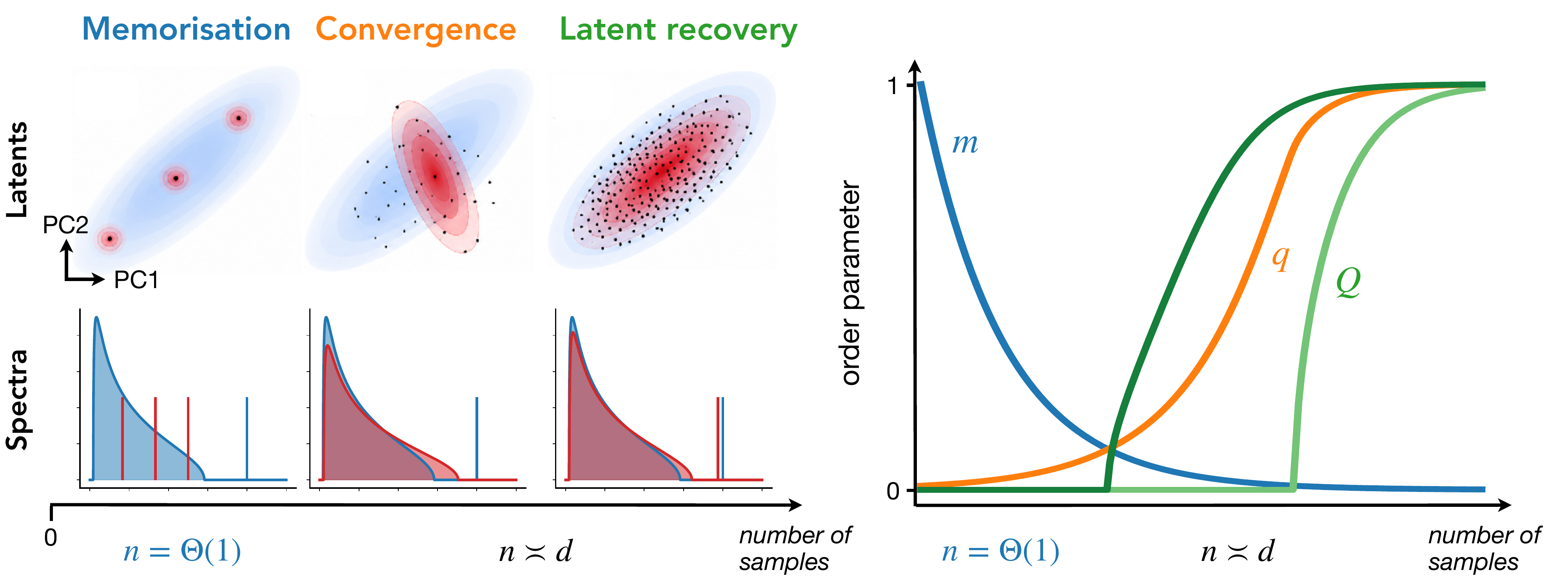}
  \caption{\label{fig:figure1} \textbf{Memorisation, convergence, and latent
  recovery in generative models.} \emph{Left}: We illustrate three regimes in
  generative models by sketching the principal components (top) and the
  covariance spectra
  (bottom) of the true data distribution (blue) and a distribution learnt from
  $n$ samples (red) in~$d$ dimensions. For very small data sets, $n=\Theta(1)$,
  the model does not learn a smooth approximation of the data distribution, but
  memorises the training set. As the number of training samples becomes
  comparable to the input dimension, $n \asymp d$, the spectra continuously
  converge and independently trained models produce similar outputs given the
  same latents, as described by \citet{kadkhodaie2024generalization}. However,
  convergence of outputs does not imply recovery of the principal latent, which
  corresponds to the outlier in the true spectrum. The principal latent is learnt
  independently from the rest of the spectrum in a sharp transition. \emph{Right}: We
  develop this picture mathematically by developing an asymptotically exact
  theory that describes the evolution of three order parameters for memorisation
  ($m$), convergence ($q$), and latent recovery ($Q$) in linear generative
  models. The two lines for $Q$ correspond to high and low signal-to-noise ratio $\beta$
  of the principal latent (dark and light green, respectively).}
\end{figure*}

The thought-provoking analysis of KGSM raises a
number of questions: how do diffusion models transition from memorisation to
generalisation? How much data is needed for convergence, and which parts of
the data distribution drive convergence? And most importantly: what does
convergence actually imply about learning the data distribution? These questions
have started to attract significant interest. A number of recent studies have
highlighted that the train and test losses of generative models are by
themselves poor proxies for the quality of the generative
models~\citep{garnier2026biased, mendes2026solvable}. Several studies analysed
the transition from memorisation to generalisation in associative
memories~\citep{pham2025memorization, kalaj2025random}, in energy-based
models~\citep{catania2025theoretical}, and in diffusion models
\emph{during training} by comparing the training and test loss both
empirically~\citep{favero2025bigger} and in random feature
models~\citep{bonnaire2025why, george2026denoising}. More recently,
\citet{wang2026random} analysed the variance of linear denoisers trained over
different realisations of a finite dataset, which has
implications for the convergence that we discuss in \cref{sec:q}. However, they
did not study memorisation, nor did they analyse what convergence implies about
generalisation.

Here, we provide an asymptotically exact analysis of the transition from
memorisation to generalisation in the simplest generative models based on linear
and spectral methods. The resulting picture, summarised as a cartoon in
\cref{fig:figure1}, separates three phenomena. With very little data, the model
memorises individual training examples. As the number of samples becomes
comparable to the input dimension, memorisation vanishes and independently
trained models converge: they produce similar outputs from the same latent
variable, as observed by \citet{kadkhodaie2024generalization}. Crucially,
convergence is driven primarily by learning the bulk spectrum of the data and
need not imply recovery of the principal latent factors, which correspond to the
outlier in the true spectrum (middle column in \cref{fig:figure1}). The latents are instead
learnt through a sharp transition. We formalise these notions by introducing
order parameters capturing memorisation ($m$), convergence ($q$), and latent
recovery ($Q$), and compute how they change with data set size in
\cref{sec:pca}. We show that convergence and latent recovery capture different
aspects of generalisation by relating them to different statistical distances
between true and learnt data distributions in \cref{sec:distances}. Finally, in
\cref{sec:powerlaw} we extend the theory to power-law spectra and define a
modified criterion for latent recovery that exhibits a sharp cross-over in
diffusion models trained on real images.

\section{Memorisation and generalisation in linear generative models}%
\label{sec:pca}

To study the transition from memorisation to generalisation quantitatively, we
have to make an explicit assumption for the data distribution. We consider the
simplest setting that captures this transition: zero-mean inputs drawn from a
multivariate Gaussian distribution. To model correlations, we use a spiked
Wishart model from random matrix theory, which captures the intrinsically
low-dimensional structure of real-world data, where a small number of latent
factors induces correlations across many observed
variables~\citep{udell2019big}. Examples include global modes in images like
illumination and pose~\citep{hyvarinen2009natural},
or social networks, where communities induce correlated connectivity
patterns~\citep{holland1983stochastic, karrer2011stochastic,
lesieur2017constrained, abbe2018community}. 
Specifically, the covariance matrix is a low-rank perturbation of the identity
in $d$ dimensions~\citep{johnstone2001distribution,potters2020first},
\begin{equation}
  \label{eq:spiked-cov}
  \Sigma^\star = \Id_d + \beta \bu \bu^\top, 
\end{equation}
where $\beta\ge0$ is the signal-to-noise ratio and the ``spike''
$\bu\in\bbS^{d-1} \coloneqq \{\bv \in \bbR^d \, : \, \|\bv\|_2 = 1\}$
represents the dominant latent direction that organises the data. This
corresponds to a limiting case of \cref{fig:figure1} where the ``bulk'' part of
the blue distribution collapses to a Dirac delta distribution located in $1$. A useful
example is a symmetric mixture of two Gaussians with identity covariance and
mean $\pm \bu$, which results in inputs with covariance $\Sigma^\star$. A more
realistic example is a social network modelled by a stochastic block
model~\citep{holland1983stochastic}, where the entries of $\bu$ encode the group
of each node in the network.

Given a data set $\mathcal{D} = \{\bx_1, \bx_2, \ldots, \bx_n\}$ of $n$ samples
$\bx_\mu \in \reals^d$ sampled i.i.d.\ from $\mathsf{P}^\star=\mathsf{N}(0,
\Sigma^\star)$, we can build a simple (Gaussian) generative model that generates new
samples $\tilde \bx (\bz)$ given a latent variable $\bz \sim N(0, \Id_d)$ as
\begin{equation}
  \label{eq:cov_emp}
    \tilde \bx(z) = \hat{\Sigma}^{\nicefrac{1}{2}} \bz, \qquad \hat{\Sigma}(\mathcal{D}) = \frac{1}{n} \sum_{\bx^\mu\in\mathcal{D}} \bx_\mu \bx_\mu^\top + \sigma^2 \Id_d,
\end{equation}
where $\hat{\Sigma}^{\nicefrac{1}{2}}$ is the matrix square root of the empirical
covariance matrix, which we estimate using a regularised (or shrinkage)
estimator with regularisation strength $\sigma \geq 0$. The key control parameter
for our analysis is the sample complexity, or load of the generative model,
which is defined as $\gamma \coloneqq n / d$. All quantities like overlaps or statistical
distances are (implicitly) a function of the sample complexity.

Linear neural networks have long been a workhorse of deep learning
theory~\citep{baldi1989neural, saxe2014exact, saxe2019mathematical,
nam2025position}. For diffusion
specifically, \citet{merger2025generalization} and \citet{wang2025analytical}
analysed the dynamics of the test loss of linear denoisers trained on the empirical
and on the population loss, respectively, but they did not study memorisation or
convergence of independently trained models, which is our focus here. More
recently, \citet{wang2026random} quantified how finite datasets shape the
expectation and variance of a learnt linear denoiser; we relate their analysis
to our convergence result in \cref{sec:q}.

\subsection{Linear models memorise inputs at small sample complexity}%
\label{sec:m}

For a generative model trained on a given data set, we quantify memorisation via
the maximal normalised overlap between a sample generated from the latent
vector~$\bz$ and the training samples, averaged over the latent vector $\bz$:
\begin{equation}
  \label{eq:m}
  m \coloneqq \EE_\bz \max_\mu  \frac{|\langle \bx_\mu, \tilde{\bx}(\bz)\rangle|}{\norm{\bx_\mu} \norm{\tilde{\bx}(\bz)}}.
\end{equation}
The first insight is that a single sample is perfectly memorised: for $n = 1$,
$m = 1$. In
\cref{app:sec:pca_memorisation}, we show that in the limit where $n \gg 1$
\emph{after} $d \to \infty$, the memorisation overlap decays as
\begin{align}
  m \simeq \sqrt{\frac{2 \log n}{n(1 + \sigma^2)}}.
\end{align}
We note the decay of memorisation with increasing data set size, and how the
regularisation $\sigma^2$ reduces memorisation. In particular, the transition
of $m$ from $1$ (perfect memorisation) to almost $0$ (near absent memorisation) occurs on the scale $n = \Theta(1)$, which we plot in blue in
\cref{fig:figure1}.

\subsection{Independently trained linear models converge at linear sample complexity}%
\label{sec:q}

To quantify the convergence of the generative models,
KGSM computed the average overlap between pairs
of samples generated by two independently trained generative models from the
same latent variable~$\bz$. In our setup, given two independently drawn data
sets $\mathcal{D}^{(b)}, b \in \{0, 1\}$ of $n$ samples $(\bx^{(b)}_\mu)_{\mu=1}^n \iid
\mathsf{P}^\star$, and given a latent variable $\bz \sim \mcN(0, \Id_d)$ we define
convergence via the overlap of two samples $\tilde{\bx}^{(b)}(\bz)$ drawn from the generative models estimated
from $\mathcal{D}^{(b)}$ given the same latent~$\bz$:
\begin{equation}
  \label{eq:q}
  q \coloneqq \EE_\bz \frac{|\langle \tilde{\bx}^{(0)}(\bz), \tilde{\bx}^{(1)}(\bz) \rangle|}{\norm{\tilde{\bx}^{(0)}(\bz)} \norm{\tilde{\bx}^{(1)}(\bz)}}.
\end{equation}
For $n = \Theta(1)$ while $d\to\infty$, a simple argument shows that $q$ decays
as $q \sim 1/\sqrt{d}$, see \cref{app:theory-convergence}.
The relevant regime for
convergence of linear generative models is instead $n \asymp d$. In this regime, we
can derive an exact expression for the convergence overlap $q$ as a function of
the sample complexity $\gamma \coloneqq n / d$, namely
\begin{equation}
    \label{eq:q_analytical}
  q = \frac{1}{1 + \sigma^2}{\left(\mathbb{E}_\gamma \sqrt{\lambda + \sigma^2}\right)}^2 + \smallO(1), 
\end{equation}
where $\EE_\gamma$ denotes the average over the Marchenko-Pastur distribution
$\mu_\gamma(\lambda)$ with aspect ratio $\gamma$, which is the limiting spectral
distribution over the bulk eigenvalues $\lambda$ of the unregularised sample
covariance $\nicefrac{1}{n} \sum_{\bx^\mu\in\mathcal{D}} \bx_\mu \bx_\mu^\top$, see~\cref{app:theory-convergence} for more details.
From the plot in \cref{fig:figure1}, we see that the
convergence overlap increases continuously from zero to one in the regime $n
\asymp d$. We note that our result for $q$ is closely connected to a result in
\citet{wang2026random}, who derive an expression for the deterministic
equivalent for the expectation of a sample generated by linear diffusion model
given the noise sample $\bz$ in their eq.~(8). Computing the squared norm of
their eq.~(8) in the case $\Sigma = \Id_d$ and $\mu = 0$ yields our
\cref{eq:q_analytical}.

Taken together, the expressions for $m$ and $q$ capture the phenomenology of the KGSM
experiment: linear generative models memorise their inputs at small sample
complexity $n = \Theta(1)$. As $n$ grows to be linear in the input dimension,
memorisation $m$ vanishes and the generative models converge in their outputs
for the same latent. However, this notion of convergence only depends on the
spectral density of the \emph{bulk} of the eigenvalues: \cref{eq:q_analytical}
is identical, at leading order, to the isotropic case where
$\mathsf{P}^\star=\mathsf{N}(0, \Id_d)$. In other words, convergence of the outputs
is not sensitive to whether the most important latent direction of the data
$\bu$ has been recovered or not! More precisely, the outlier eigenvalue
corresponding to the spike $\bu$ contributes only to the vanishing $o(1)$
correction associated with the single spiked eigenvalue, because a rank-one
spike affects only one out of $d$ eigenvalues. In fact, as $\gamma \to \infty$,
the Marchenko-Pastur bulk concentrates at $\lambda=1$, so~$q \to 1$. To capture
the recovery of the leading latents on generalisation, we now introduce another
order parameter.

\subsection{Subspace recovery and the BBP transition}%
\label{sec:Q}


The failure of the convergence overlap $q$ to capture recovery of the principal latent
factor leads us to introduce a third overlap~$Q$ that
specifically measures the alignment between the two top eigenvectors of
$\hat{\Sigma}^{(1)}$ and $\hat{\Sigma}^{(2)}$,
\begin{equation}
  \label{eq:Q}
    Q \coloneqq |\langle \bv_{\max}(\hat{\Sigma}^{(0)}), \bv_{\max}(\hat{\Sigma}^{(1)}) \rangle |.
\end{equation}
This order parameter measures to what extent the two generative models agree on
what is the principal subspace, and hence the principal latents, of the data.
We compute an analytical expression for $Q$ using standard results from
random matrix theory for Wishart matrices, and in particular the so-called
``BBP''-transition of \citet{baik2005phase}, first described in \citet{edwards1976eigenvalue}:
\begin{align}
  \label{eq:Q_delta}
  Q = \delta(\gamma)^2 + \smallO(1), \qquad
  \delta(\gamma) \coloneqq 
  \begin{dcases}
    0 & \textrm{ if } \gamma < \gamma_c(\beta) \coloneqq \beta^{-2}, \\
    \sqrt{\frac{\beta^2 \gamma - 1}{\beta(1 + \beta \gamma)}} & \textrm{ if }
    \gamma \geq \gamma_c(\beta),
  \end{dcases}
\end{align}
where $\delta(\gamma)$ is obtained from the the BBP
formula~\cite{benaych2011eigenvalues}, see \cref{app:sec:bbp}.  Notice that
indeed $Q \to 1$ as $\gamma \to \infty$, but that it only becomes non-trivial after $\gamma$ crosses a $\beta$-dependent threshold, as opposed to the convergence parameter $q$. We plot the resulting expression for
$Q$ in green in \cref{fig:figure1} for small and large values of $\beta$ (light
and dark green, respectively). We see that the principal subspace can be
recovered before or -- more concerningly when using convergence as a proxy for
generalisation -- \emph{after} convergence. The key point we wish to highlight
is that convergence and latent recovery occur \emph{independently of each
other}. Note that the different behaviours of convergence and subspace recovery would
equally apply to a data model in which there are several spikes $\bu_k$, as long
as their number remains of order 1 with respect to the input
dimension~\cite{benaych2011eigenvalues}. These behaviours hint at different
notions of generalisation that might be appropriate on data with low-dimensional
latent structure; we now make this intuition precise by relating the overlaps
$q$ and $Q$ to different statistical distances between probability densities.

\section{Convergence and latent recovery govern different statistical distances
between true and learnt distributions}%
\label{sec:distances}

So far, we have characterised the learnt model using the order parameters $m$,
$q$, and $Q$, which may appear somewhat \emph{ad hoc} as measures for the
quality of the learnt distribution. Ideally, we would like to assess the quality
of the learnt model $\hat{\mathsf{P}}$ via its statistical distance to the
true distribution $\mathsf{P}^\star$. Here, we show that~$q$ and $Q$ are
indeed proxies for statistical distance: specifically, they correspond to the
Kullback-Leibler divergence and the maximum-sliced (MS) distance, respectively.
This distinction highlights the different learning processes they capture.

\begin{figure}[t!]
  \centering
  \includegraphics[width=.4\linewidth]{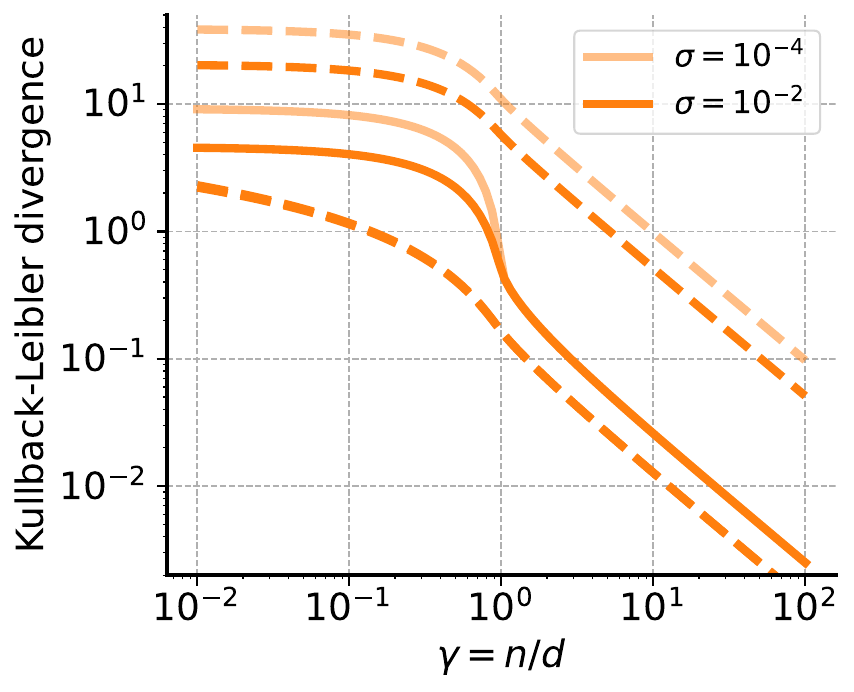}%
  \includegraphics[width=.4\linewidth]{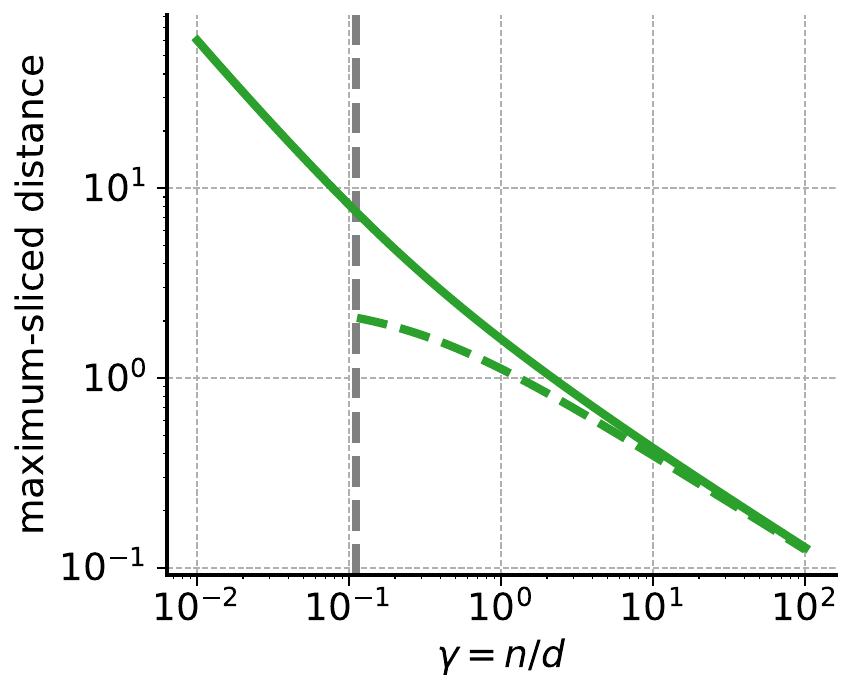}
  \caption{\label{fig:distances} \textbf{Convergence and latent recovery
  govern different statistical distances between true and learnt
  distributions $\mathsf{P}^\star$ and $\hat{\mathsf{P}}$.} \emph{Left}: Rescaled Kullback-Leibler divergence between
  $\mathsf{P}^\star$ and $\hat{\mathsf{P}}$
  for two regularisation strengths $\sigma$, as a
  function of sample complexity $\gamma = n / d$, where $n$ is the number of
  samples and $d$ is the input dimension. The dashed lines show the upper and lower
  bounds of \cref{eq:dkl_bounds}, which only depend on the
  convergence overlap $q$. \emph{Right}: Maximum-sliced distance between $\mathsf{P}^\star$ and $\hat{\mathsf{P}}$ (solid line), which decays asymptotically with the subspace
  overlap as $\sqrt{1-Q}$, see 
  \cref{eq:dms-asymptotics} (dashed line). \emph{Parameters: $\beta=4$}.}
\end{figure}

\subsection{The convergence overlap $q$ controls the Kullback-Leibler divergence}

The Kullback-Leibler divergence $\DKL$ 
between true and learnt distributions is not affected by the presence of
a spike in the covariance, so to leading order we can consider the case $\beta =
0$, see \cref{app:sec:what-spike}.  To remove the slight mismatch between true and
learnt distribution due to the regularisation,
we consider an ``excess'' $\DKL$ between the learnt and the ``regularised''
true covariance $\Sigma^\star + \sigma^2 \Id_d$. In
\cref{app:sec:q_DKL_proof}, we show that the convergence overlap $q$ provides
the following upper and lower bounds of the rescaled Kullback-Leibler divergence
$\mathcal{D} \coloneqq \nicefrac{1}{d} \; \EE \, \DKL \left[\mathsf{N}(0,
\hat{\Sigma}) | \mathsf{N}(0, \Sigma^\star + \sigma^2 \Id_d) \right]$, 
\begin{equation}
  \label{eq:dkl_bounds}
  -\frac{1}{2} \log q \leq \mcD \leq 2 \left[1 + \log\left(1 + \sigma^{-2}\right)\right] (1-q).
\end{equation}
We plot $\mathcal{D}$ for two different
values of the regularisation $\sigma$ together with the bounds in
\cref{fig:distances}. There are two regimes: when the number of samples is smaller than the input dimension $n < d$, $\DKL$ is almost
stationary and its value determined by the regularisation parameter $\sigma^2$;
for $n>d$ instead, it decays with the sample complexity as a power-law, which is
nicely captured by the two bounds. 
We note that for $\gamma \geq 1$ one can improve the
multiplicative constant $1 + \log(1+\sigma^{-2})$ to an explicit value that goes to $2$ for large $\gamma$ and any finite $\sigma \geq 0$, see the end of~\cref{app:sec:q_DKL_proof}.

The proofs of the bounds do not explicitly use the form of the Marchenko-Pastur law; instead, we only require that $\mathcal{D}$ can be
expressed in terms of the spectral density. Outside of the elliptical family of
distributions, which contains the multivariate normal distribution, this direct
link can break. For example, if two estimators capture the same subset of modes
of a multi-modal data distribution, samples from the estimators can be
very close, even though the $d_\mathrm{KL}$ to the true distribution
will be large due to the missing modes. This simple example reveals another
caveat of using convergence as a measure of generalisation.

\subsection{Latent recovery governs the decay of the maximum-sliced distance}


In contrast to the Kullback-Leibler divergence, where we integrate over the
whole input space, the maximum-sliced distance
$\MS$~\cite{kolouri2019generalized} uses a smooth test function to measure the
largest discrepancy between two multivariate distributions over all linear
projections. We obtain an exact expression for $\MS(\mathsf{P}^\star,
\hat{\mathsf{P}})$ as $n, d \to \infty$ in \cref{app:sec:MS-Q}, and plot the
result in \cref{fig:distances}. For $\beta \ge 1 + \gamma^{-\nicefrac{1}{2}}$, the decay of $\MS$ is driven by the increase of subspace overlap $Q$; as $Q
\to 1$, the $\MS$ goes to 0 as
\begin{equation}
  \label{eq:dms-asymptotics}
  \MS(\mathsf{P}^\star, \hat{\mathsf{P}}) = \frac{1}{2} \sqrt{\beta (1 + \beta)} \sqrt{1 - Q} + \mcO(1 - Q) \quad \mathrm{as}\quad Q\to1.
\end{equation}
We plot this asymptotic expression for $\MS$ with the dashed green line in
\cref{fig:distances} to show how subspace recovery directly governs convergence
of the learnt distribution in terms of maximum-sliced distance, validating the
choice of $Q$ as a measure of generalisation.

\section{Towards realistic data: the case of power-law spectra}%
\label{sec:powerlaw}

\subsection{Power-law spectra confound convergence and subspace recovery}%

How does our theory extend to real images? 
A key feature of images that is not included in the discussion above is that their
spectra, i.e.\ the eigenvalues of their covariance, decay like a power-law. This
decay reflects a continuous hierarchy of modes rather than a clear separation
between a bulk and spikes~\citep{hyvarinen2009natural}. For example, the
eigenvalues of the empirical covariance of images from the CelebA data set and
of patches sampled randomly from ImageNet decay (roughly) as a power-law,
see \cref{fig:powerlaw}(a). 
However, this decay is not visible trivially in the space of pixels: the variances of individual pixels
$\hat{\sigma}_i^2 \coloneqq \hat{\Sigma}_{ii}$ are comparable across the whole image.

This power-law structure in real images drastically changes the behaviour of the convergence overlap~$q$ and the subspace recovery parameter $Q$. We plot both order parameters for
linear generative models ``trained'' on increasingly large, disjoint subsets of
CelebA in \cref{fig:powerlaw}(b), see \cref{app:sec:datasets} for experimental
details. Contrary to the near-isotropic setting of \cref{sec:pca}, where $Q$
underwent a sharp phase transition, on CelebA $Q$ attains large values for a
small number of samples and continues to increase continuously from there, and
likewise on ImageNet, see \cref{fig:imagenet}. The rapid increase of $Q$ on real
images is driven by low-frequency modes, whose large eigenvalues make them easy
to recover. 
Informally, this phenomenon is a consequence of the generative models agreeing on the power-law decay, rather than the latent structure of the data. 
We analyse this point theoretically below, by introducing a 
statistical model for data that correctly captures the pixel-wise variance and
spectral density of real images.
Crucially, our analysis suggests a refinement of
the subspace overlap that $(i)$ is appropriate for real images, and $(ii)$ undergoes a sharp transition corresponding to the recovery of the latent structure of the data. 
We thus redeem the distinction between convergence and subspace recovery in the context of real images.

\begin{figure*}[t!]
  \centering
  \includegraphics[width=.33\linewidth]{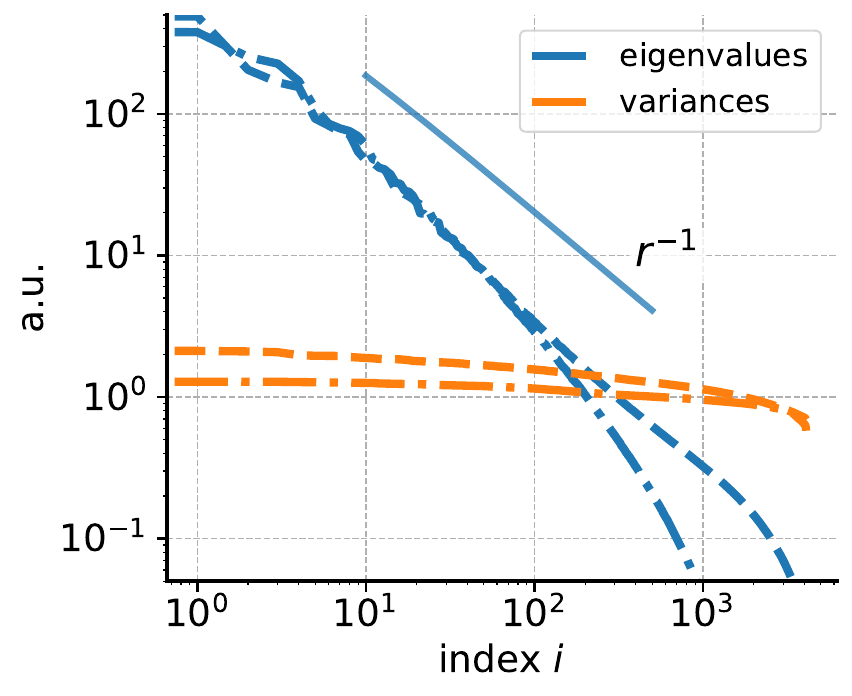}%
  \includegraphics[width=.33\linewidth]{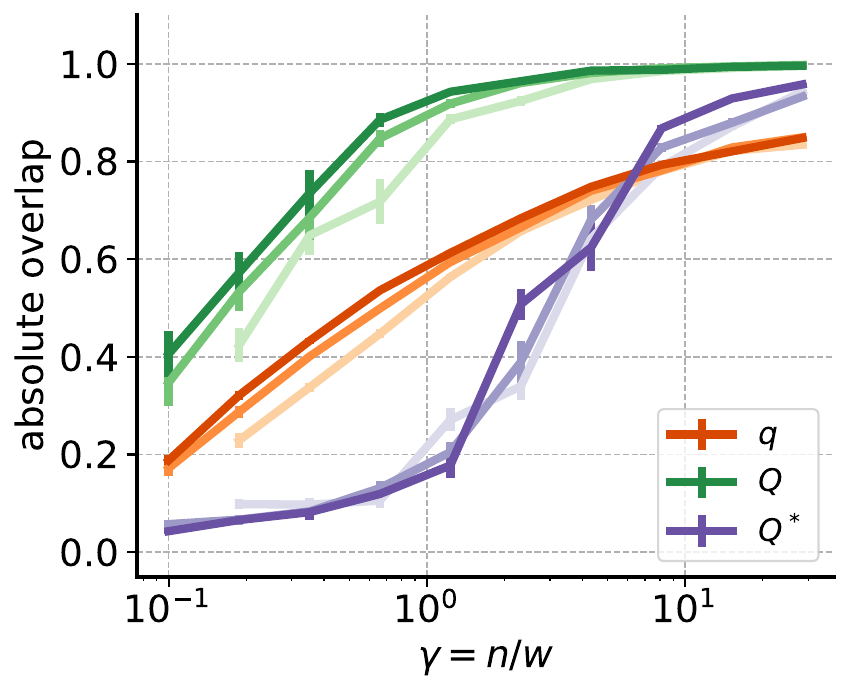}%
  \includegraphics[width=.33\linewidth]{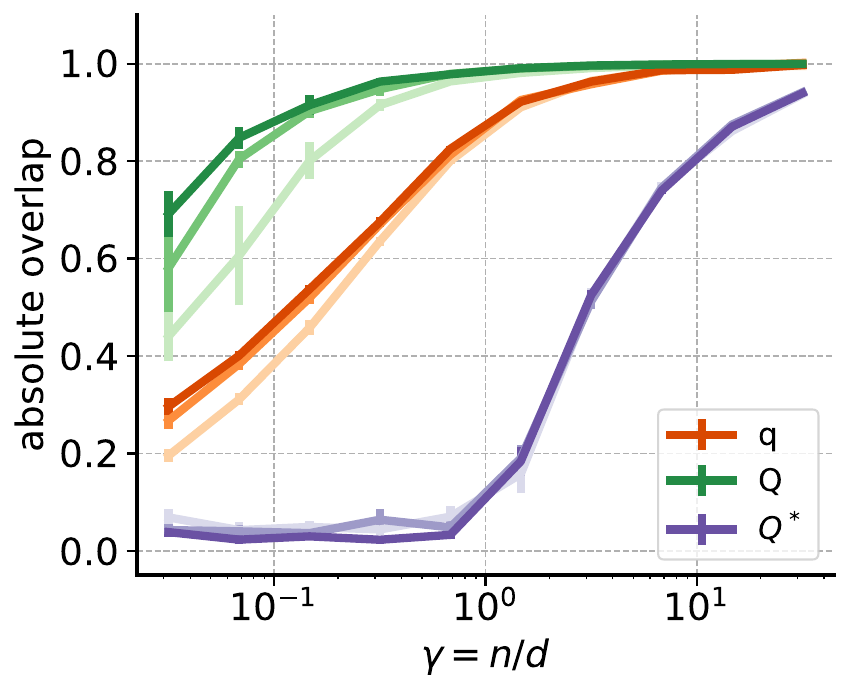}%
  \caption{\label{fig:powerlaw} \textbf{A rotated subspace overlap identifies a
  phase transition on data with power-law spectrum.} \textbf{(a)}:
  Eigenvalues~$\lambda_i$ and variances $\hat{\sigma}_i^2 = \hat{\Sigma}_{ii}$
  of the covariance matrix $\hat{\Sigma}$ of CelebA images (dashed-dotted) and
  ImageNet patches (dashed) against their index $i$ (image width $w=64$).
  \textbf{(b)} Convergence overlap $q$ and subspace overlap $Q$,
  \cref{eq:q,eq:Q}, for a linear model fitted to CelebA images as a function of
  sample complexity $\gamma = n / w$ for image widths $w \in \{16, 32, 64\}$
  (increasingly dark lines). We also plot the rotated subspace overlap $Q^\star$
  that factors out the power-law decay of the spectrum, \cref{eq:Qstar}.
  \textbf{(c)} The same order parameters as in (b) evaluated
  on the power-law data model from \cref{sec:powerlaw-model} with $\beta=1$.}
\end{figure*}

\subsection{The power-law input model}
\label{sec:powerlaw-model}

The key idea of this ``power-law'' input model is to separate the effects of the
latents from the power-law spectrum of the covariance. We work in a fixed
orthonormal basis $F \in \mcO(d)$. For images, the Fourier basis is a natural
choice since Fourier amplitudes follow a power-law, but for our model any
delocalised basis, e.g.\ any uniformly-sampled random orthogonal matrix $F$, will
do; an extension to unitary rather than orthogonal bases is also
straightforward.  We sample a sequence of amplitudes $\{\sigma_k\}_{k=1}^d$ from
a power-law distribution
 \begin{equation}
  \label{eq:eigenvalues}
  \sigma_k \propto k^{-\alpha}, \qquad \sum_{k=1}^d \sigma_k^2 = 1,
\end{equation}
where $\alpha > 0$ is the exponent that governs the decay of the spectral
density, and the normalisation condition controls the norm of the resulting inputs (akin
to normalising the total brightness of images). 
Finally, we introduce a latent structure by assuming that the correlation
matrix of our zero-mean inputs, $\rho_{ij} \coloneqq \EE[x_i x_j] /
\sqrt{\EE[x_i^2] \EE[x_j^2]}$, has a rank-one structure similar
to \cref{eq:spiked-cov}. All in all, we arrive at the following ``power-law
model'' for the covariance matrix of the data distribution:
\begin{equation}
  \label{eq:powerlaw-model}
  \Sigma^\star = F^\top D_{\sigma} \rho D_{\sigma} F \quad \text{with} \quad   \rho = \Id_d + \beta \bu \bu^\top,
\end{equation}
where $D_{\sigma} \coloneqq \textrm{Diag}(\{\sigma_k\}_{k=1}^d)$.

The power-law model reproduces several features seen in the distribution of real images.
First, it reproduces
the power-law spectrum of the covariance and the near-isotropic variances of
real images that we saw in \cref{fig:powerlaw}(a). Second, we show in
\cref{fig:image_correlations} that the spectrum of the correlation matrix of CelebA in
Fourier space does indeed separate into a bulk and a few outlier eigenvalues. This finding validates the rank-one additive model for the correlation matrix in the rotated space in~\cref{eq:powerlaw-model}. 
Finally, \cref{fig:powerlaw}(c) reveals that the
evolution of $q$ and $Q$ for linear generative models ``trained'' on inputs
sampled from the power-law input model of \cref{eq:powerlaw-model} closely matches the results on real
images; in particular, the model reproduces the quick rise of $Q$.

The power-law model identifies the reason for which the subspace overlap $Q$ fails to detect
consistent estimation of the latent structure of the data. Within the model, the
spike $\bu$ is the most relevant latent factor because it identifies which
variables fluctuate together because of the shared latent, independently of
whether some variables are intrinsically larger or smaller than others because
of the power-law decay of the covariance eigenvalues. However, the eigenvectors
of~$\Sigma^\star$ do not coincide with $\bu$; instead, a short calculation
reveals that the leading eigenvector $\bv$ with eigenvalue
$\lambda$ of the population covariance $\Sigma^\star$ is given by
\begin{equation}
  \label{eq:rotated-eigenvector}
  \bv = F^\top \tilde{\bv} \qquad \text{with} \qquad \tilde{v}_i \propto \frac{\sigma_i u_i}{\lambda - \sigma_i^2}
\end{equation}
Hence in the $F$-basis, the spike enters as $D_\sigma \bu$, so coordinates with
larger $\sigma_i$ are amplified, while those with smaller $\sigma_i$ are
suppressed. As a result, the leading principal component (the top eigenvector of
the empirical covariance matrix) mixes latent structure with the variance
profile, and may be dominated by the latter. Hence $Q$ is strongly confounded by
the variance profile and can be large even when recovery of the latent
correlation structure remains weak, as seen in \cref{fig:powerlaw}.

\subsection{A rotated subspace overlap undergoes a sharp transition on real
images}%

\begin{figure*}[t!]
  \centering
  \includegraphics[width=\linewidth]{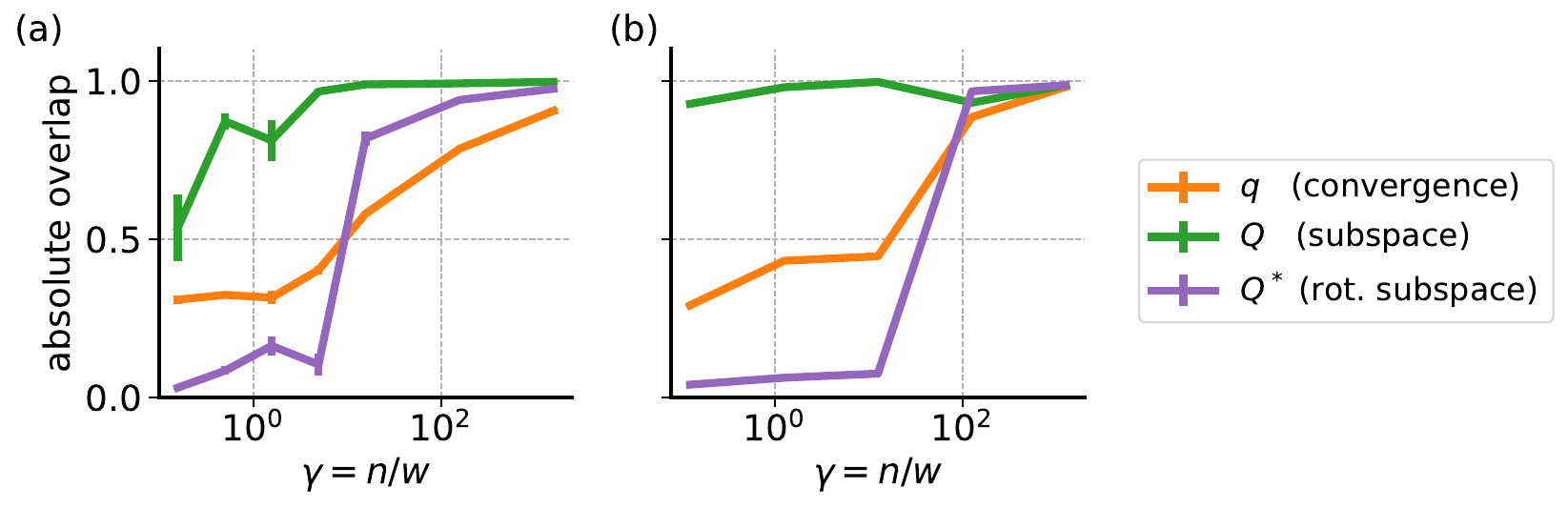}
  \caption{\label{fig:diffusion} \textbf{Convergence and subspace recovery in
  diffusion models.} We plot the convergence overlap $q$, the naïve subspace
  overlap $Q$ and the rotated subspace overlap $Q^\star$ in the Fourier basis, \cref{eq:q,eq:Q,eq:Qstar}
  respectively, for images sampled from diffusion models trained on CelebA.
  \textbf{(a)} Images sampled from a vanilla UNet denoiser that we trained
  independently on increasing large disjoint subsets of CelebA. Average and
  error in the mean are computed over four runs per sample complexity, see
  \cref{app:sec:diffusion} for details. \textbf{(b)} Images sampled by
  \citet{kadkhodaie2024generalization} from a blind
  denoiser~\citep{mohan2020robust} using the sampling algorithm of
  \citet{kadkhodaie2020solving}.}
\end{figure*}

How can we ``factor out'' the effect of the power-law spectrum? A natural idea
is to consider the leading eigenvector of the correlation matrix of the inputs.
However, since the pixel variances are all comparable, the correlation matrix is
similar to the covariance matrix, and yields similar results. The key is instead
to analyse the data in a basis where the power-law structure is explicit, like
the $F$ basis in our model or the Fourier basis for natural images. In this
representation, the dominant spectral modes can be factored out by working with
the correlation matrix, isolating the low-rank structure. We therefore define
the \textbf{rotated subspace overlap $Q^\star$} as the overlap between the leading
eigenvectors of the \emph{correlation} matrix of inputs in the rotated space,
\begin{equation}
  \label{eq:Qstar}
    Q^\star \coloneqq |\langle \bv_{\max}(\hat{\rho}^{(0)}), \bv_{\max}(\hat{\rho}^{(1)}) \rangle |,
\end{equation}
where $\hat{\rho}^{(0,1)}$ are the correlation matrices of the inputs
\emph{after} the linear transformation.
Note that $Q^\star$ depends on the choice of the basis in which we rotate the inputs.
As we can see in \cref{fig:powerlaw}(c),
$Q^\star$ (violet) undergoes a BBP-like phase transition around a sample complexity $\gamma = 1$ in the power-law model of \cref{eq:powerlaw-model} with $\beta  = 1$. On CelebA, we evaluate $Q^\star$ by 
computing the overlap between the leading
eigenvectors of the (complex) empirical correlation matrices between Fourier
coefficients of two disjoint data sets. Again, we find that $Q^\star$ undergoes a
sharp transition as we increase the sample complexity~$\gamma = n / w$, where
the normalisation is now with respect to the patch width $w$ (violet line in
\cref{fig:powerlaw}(b) and in \cref{fig:imagenet} for ImageNet).

To validate our choice of $Q^\star$ as a probe for latent correlation structure, we
considered controls in which this structure is removed while the spectral
envelope is preserved. In the power-law model of \cref{eq:powerlaw-model},
setting $\beta=0$ removes the correlation spike $\bu$ without changing the
spectrum; we show in \cref{fig:powerlaw_snr0}(a) that this leaves $q$ and~$Q$
unaffected, while $Q^\star$ remains close to zero for all~$\gamma$. On CelebA, we
can obtain the same effect by randomising the phase of each Fourier-coefficient
of each image independently, which destroys the off-diagonal correlations in
Fourier space while maintaining the power-law spectrum, see
\cref{app:sec:powerlaw_snr0}. We found that on CelebA images with randomised
phases, $Q^\star$ remains close to zero, while $q$ and $Q$ remain unaffected, see
\cref{fig:powerlaw_snr0}(b).

We also evaluated $Q^\star$ on images sampled from trained diffusion models.
We trained these diffusion denoising probabilistic models~\citep{sohl2015deep,
ho2020denoising} using a U-Net denoiser on increasingly large subsets of CelebA,
mirroring the approach of KGSM. While they trained a ``blind'' denoiser that is
not given the noise level explicitly~\citep{mohan2020robust} and sampled from it
using the algorithm of \citet{kadkhodaie2020solving}, we used a vanilla denoiser
and a standard sampling procedure (see \cref{app:sec:diffusion} for details on
the experimental setup). We then computed the overlaps $m, q, Q$ and $Q^\star$ from
images sampled from these models. We first confirmed the transition from
memorisation to convergence as we increased $n$, see
\cref{fig:diffusion_m_and_q}. In \cref{fig:diffusion}, we find that while the
``naïve'' subspace overlap $Q$ is un-informative on the generated samples, which
have a power-law spectrum, the rotated subspace overlap $Q^\star$ exhibits a sharp
transition both when evaluating on the images generated from the models that we
trained (left) and on the images provided by KGSM (right).

\section{Discussion: the two sides of generalisation in generative models}%
\label{sec:discussion}

Our theoretical analysis of simple generative models revealed a clear separation
between memorisation, convergence, and subspace recovery. Linear generative
models memorise their training samples at low load $n = \Theta(1)$ and
continuously converge in output conditioned on the same latent at linear sample
complexity $n \asymp d$. However, this convergence is \emph{not} sensitive to
consistent estimation of the principal latent factors of the data, which is
captured by the subspace overlap $Q$. The latter quantity undergoes a BBP-style phase
transition~\citep{baik2005phase} which can occur both before or -- more
concerningly when using convergence as a proxy for generalisation --
\emph{after} convergence. For data with power-law spectra, a rotated subspace
overlap defined in terms of the leading eigenvectors of Fourier-space correlations recovers the same
phenomenology on real images.

Instead of identifying the ``right'' measure of generalisation, our analysis
highlights two different objectives for learning a generative model. The
convergence overlap $q$ captures agreement in the ``bulk'' of the learnt
distribution, while the subspace overlap $Q$ determines whether leading latent
factors of the data have been estimated consistently. In our setting, this
distinction can be made mathematically precise by relating $q$ to the
Kullback-Leibler divergence and $Q$ to the maximum-sliced distance between true and learnt distributions. The
separation between the two may be even stronger in tensor PCA~\citep{montanari2014statistical} or non-Gaussian methods like
independent component analysis (ICA)~\citep{hyvarinen2000independent}, where
spike recovery in high-dimensions requires at least $n\gtrsim d^2$
samples~\citep{ricci2025feature}.

The distinction between learning the bulk and the latents goes beyond the realm of images. Take social
networks for example. In the stochastic block
model~\citep{holland1983stochastic, lesieur2017constrained, abbe2018community},
the community structure appears as a low-rank perturbation to an otherwise
random connectivity, akin to the rank-one perturbation of the identity
covariance in the spiked Wishart model. In this setting, convergence and
subspace recovery again separate two distinct objectives of modelling: learning
the bulk statistics of the graph -- for example, is the number of connections a
node makes distributed isotropically or does it follow a power-law? -- vs.\
learning the community structure that creates the low-rank signal in the
connectivity. A generative model can ``fit the bulk'' by matching the statistics
of the underlying graph, while still missing the communities, which would
translate into large~$q$ and small $Q$. 

Social networks also highlight the relevance of the rotated subspace overlap
beyond the realm of images. Real social networks typically exhibit strong degree
heterogeneity with a few highly connected nodes and many weakly connected ones,
a phenomenon known as Feld's friendship paradox: ``your friends have more
friends than you.''~\citep{feld1991your}. This heterogeneity can dominate the
leading spectral modes of the connectivity matrix, so that the top eigenvector
separates high- and low-degree nodes while missing the underlying community
structure. As a result, the subspace overlap $Q$ can be large even when the
latent structure, i.e.\ the community assignments, are not yet recovered. As in
images with power-law spectra, one therefore needs a whitened or rotated notion
of subspace recovery that factors out the degree heterogeneity before probing
consistent recovery of the low-rank signal~\citep{karrer2011stochastic}.

The idea of comparing generative models trained on disjoint subsets of the same
data set has an interesting connection to classical statistical learning theory
(SLT)~\citep{hastie2009introduction}. SLT relates the ability of a machine
learning model to generalise to how much the trained model changes when one
training sample is removed. In a classic analysis, Vapnik and Chervonenkis
showed that if changing one data point rarely changes the behaviour of a trained
perceptron, its \emph{test} error is small~\citep{vapnik1974theory}. In a
similar spirit, KGSM argue that if the learnt denoiser is robust to resampling
the \emph{entire} training set, it is a function of the underlying data
distribution, rather than the training set, and should hence generalise well.
Exploring this connection in detail is an interesting direction for future
research.

\citet{kadkhodaie2024generalization} and \citet{kamb2025analytic} made the point
that the inductive biases of convolutional denoisers are crucial in achieving the transition from memorisation to
generalisation. These inductive biases may
affect bulk fitting and latent recovery in different ways. For example,
architectural choices that promote smoothness and locality may accelerate
convergence without necessarily improving recovery of the latent factors.
Characterising the impact of different inductive biases on different objectives is
an important direction for future work.

Generally speaking, our work suggests that generalisation in generative
models is best understood as a sequence of transitions associated with
distinct statistical objectives. Hence evaluation of generative models should
also consider several objectives~\citep{vandenburg2021memorization}, a point
that was recently also made by \citet{garnier2026biased,mendes2026solvable}.
The distinction between learning bulk statistics like spectra and learning the
principal latent variables is especially important for the growing number of
applications
of generative models in the natural sciences. A model that fits the bulk
while missing the principal latents may perform well as a sampler, while being
misleading as a scientific model.


\section{Concluding perspectives}

Our analysis of linear generative models supports the idea that generalisation
in generative models is a multi-faceted concept. An important next step is to
determine how far the picture we have developed here extends beyond the simple
setting that we have studied. We considered a static estimator for our
generative model; analysing the dynamics of $m, q$ and $Q$ in simple models and
real diffusion models is an important next step.  We made extensive use of the
Gaussian approximation for both the data and the generative model, but real data
is not Gaussian, and good generative models capture these higher-order
correlations: extending our work beyond the Gaussian approximations is an
interesting avenue for further research. A natural starting point would be
non-Gaussian linear generative models like ICA~\cite{hyvarinen2000independent},
or to use ideas from Gaussian universality, both for kernel
matrices~\citep{liao2018spectrum, seddik2019kernel, mei2021generalization} in
Kernel PCA~\citep{scholkopf1997kernel} and for simple neural
networks~\citep{goldt2020modelling, hu2022universality, goldt2021gaussian,
loureiro2021capturing}. In terms of denoisers, an extension to either small
auto-encoders~\citep{refinetti2022dynamics, cui2023high, cui2024analysis,
bardone2026theory} or random feature models~\citep{bonnaire2025why,
george2026denoising} could reveal how learning non-Gaussian input statistics
affects memorisation and generalisation. The main control parameter in our
analysis was the size of the training data set. Taking a complementary
perspective, \citet{biroli2024dynamical} established a series of phase
transitions that a sample undergoes \emph{during sampling} in a trained diffusion
model, see also refs.~\citep{george2025analysis, achilli2026theory}.
Establishing the dynamics of convergence and latent recovery during sampling is
an interesting open problem. Finally, maybe the most important open question stemming from our work is
how to extend the idea of convergence and latent recovery to auto-regressive
generative models, and in particular to transformers for text generation.

%% file: m2g_pca_acknowledgements.tex
We thank Claudia Merger, Marc Mézard, Fabiola Ricci, Enrico Ventura, and Lenka Zdeborová for
stimulating discussions. We thank Florentin Guth, Guido Sanguinetti and Zahra
Kadkhodaie for valuable comments and 
feedback on the manuscript. SG gratefully acknowledges funding from the European
Research Council (ERC) for the project ``beyond2'', ID 101166056; from the
European Union--NextGenerationEU, in the framework of the PRIN Project SELF-MADE
(code 2022E3WYTY – CUP G53D23000780001), and from Next Generation EU, in the
context of the National Recovery and Resilience Plan, Investment PE1 -- Project
FAIR ``Future Artificial Intelligence Research'' (CUP G53C22000440006).

%% file: m2g_pca_supp.tex
\section{Memorisation and generalisation in linear generative models}%
\label{app:m2g_pca}

In this section, we give a detailed analysis of memorisation, convergence and
subspace recovery in linear generative models.

\subsection{Setting}%

\paragraph{Ground truth data distribution} We sample a data set of inputs
i.i.d.\ from the ground truth distribution $\mathsf{P}^\star$, which we take to
be a multivariate normal
distribution with mean zero and covariance $\Sigma = \Id_d + \beta \bu \bu^\T$, for some unit-normed vector $\bu \in \bbS^{d-1}$.

\paragraph{Generative model} Given a data set
$\mathcal{D} = \{\bx_1, \bx_2, \ldots, \bx_n\}$ of $n$ samples
$\bx_\mu \in \reals^d$ with zero mean, we can build a simple (Gaussian)
generative model by estimating the empirical covariance matrix directly from the
inputs:
\begin{equation}
  \label{app:eq:cov_emp}
  \hat{\Sigma}(\mathcal{D}) = \frac{1}{n} \sum_{\mu=1}^n \bx_\mu \bx_\mu^\top + \sigma^2 \Id_d,
\end{equation}
where we added a regularization $\sigma > 0$. 
We sample from the model by mapping a latent variable $\bz \sim \mcN(0, \Id_d)$ from
the latent space to the ``image space'' via
\begin{equation}
  \tilde \bx(z) = \hat{\Sigma}^{\nicefrac{1}{2}} \bz,
\end{equation}
where $\Sigma^{\nicefrac{1}{2}}$ is the matrix square root of the empirical covariance
matrix.

\subsection{Memorisation}
\label{app:sec:pca_memorisation}

For a generative model trained on a given data set, we
quantify memorisation via the maximal normalised overlap between a sample
generated from the latent vector $\bz \sim \mcN(0, \Id_d)$ and the training samples. For the $\mu$th
training sample, we have
\begin{equation}\label{eq:def_mmu}
  m^\mu(\bz) \coloneqq \frac{|\langle \bx_\mu, \tilde{\bx}(\bz)\rangle|}{\norm{\bx_\mu} \norm{\tilde{\bx}(\bz)}}.
\end{equation}
From there, we define the memorisation overlap $m$ as the maximum of sample
overlaps $m^\mu(\bz)$, averaged over the distribution of latent
vectors\footnote{Since $\max_\mu m^\mu(\bz)$ is the maximum over a set of
  Gaussian random variables with variance $1 / d$ (the overlaps), $m$ grows as
  $m \sim \sqrt{\nicefrac{2}{d} \log n}$, where $n$ is the training set, even if
  the generator is independent of the training set. It is therefore tempting to
  subtract $\sqrt{\nicefrac{2}{d} \log n}$ from the definition of $m$, but as we
  see in \cref{sec:m}, $m\to 0$ before this correction becomes
  relevant.},
\begin{equation}
  m \coloneqq \EE_\bz \max_\mu m^\mu(\bz). 
\end{equation}
Let us sketch our analysis for the behaviour of $m$ for finite $n \geq 1$ as $d \to \infty$,
as the computations follow very standard techniques.

\myskip
\textbf{Limiting behaviour of $\hSigma$ --}
Without loss of generality we can assume that $\bu = \bbe_1 = (1, 0, \cdots, 0)$.
Letting $\alpha(\beta) \coloneqq \sqrt{1+\beta} - 1 \geq 0$, we have $\sqrt{\Sigma} = \Id_d + \alpha(\beta) \bu \bu^\T$.
This allows to decompose the distribution of $\bx_\mu$ as 
\begin{equation}\label{eq:decomp_xmu}
    \bx_\mu = \bu_\mu + \alpha(\beta) u_{\mu,_1} \bbe_1,
\end{equation}
where $\bu_\mu = (u_{\mu, i})_{i=1}^d \sim \mcN(0, \Id_d)$ and $(\bu_1, \cdots, \bu_n)$ are independent.
We can decompose 
\begin{equation}\label{eq:hSigma}
    \hSigma - \sigma^2 \Id_d
    = \frac{1}{n} \sum_{\mu=1}^n \bx_\mu \bx_\mu^\T 
    = \underbrace{\frac{1}{n} \sum_{\mu=1}^n \bu_\mu \bu_\mu^\T}_{\eqqcolon \hSigma_0} + \frac{\alpha(\beta)^2}{n} \sum_{\mu=1}^n u_{\mu,1}^2 \bbe_1 \bbe_1^\T + \frac{\alpha(\beta)}{n} \sum_{\mu=1}^n u_{\mu,1} (\bu_\mu \bbe_1^\T + \bbe_1 \bu_\mu^\T).
\end{equation}
$\hSigma_0$ is the covariance matrix of $n$ standard Gaussian vectors in dimension $d$.
In the regime $d \to \infty$ while $n$ remains finite, it is a classical result
(see e.g.~\citet{vershynin2018high}, chapter 3) that 
\begin{equation*}
    \left\|\frac{n}{d} \hSigma_0 - P_U \right\|_\op \to 0
\end{equation*}
as $d \to \infty$. Here $\|M\|_\op \coloneqq \max_{\|\bu\|=1} \|M \bu\|_2$ is the operator norm, and $P_U$ is the orthogonal projector on $U \coloneqq \Span(\bu_1, \cdots, \bu_n)$, a uniformly-sampled random $n$-dimensional subspace of $\bbR^d$. 
One can easily bound the spectral norm of the two last finite-rank terms in \cref{eq:hSigma}:
all in all, we reach that, for $U$ a uniformly-sampled random $n$-dimensional subspace of $\bbR^d$, we have as $d \to \infty$:
\begin{equation}\label{eq:cv_hSigma}
    \left\|\frac{n}{d} (\hSigma - \sigma^2 \Id_d) - P_U \right\|_\op \to 0,
\end{equation}
and furthermore notice that $\Im(\hSigma - \sigma^2 \Id_d) \subseteq V \coloneqq \Span(U, \{\bbe_1\})$.

\myskip 
\textbf{Memorisation at finite $n$ --}
Let us come back to \cref{eq:def_mmu}.
Using \cref{eq:cv_hSigma}, the remark below it, and $\|\tbx(\bz)\|^2 = \bz^\T \hSigma \bz$, we reach that 
\begin{equation*}
    \left|\frac{\|\tbx(\bz)\|^2}{d} - \frac{1}{n} \bz^\T P_U \bz - \sigma^2\right| \leq \frac{\|P_{V} \bz\|^2}{n}  \left\|\frac{n}{d} (\hSigma - \sigma^2 \Id_d) - P_U \right\|_\op +  \sigma^2\left|\frac{\|\bz\|^2}{d} - 1\right|  \to 0,
\end{equation*}
where 
we used that $P_V \bz$ is a standard Gaussian vector in $V$, a subspace of dimension at most $n+1$.
Concerning the numerator in \cref{eq:def_mmu}, one can tackle it in very similar ways (using that $\sqrt{\sigma^2 \Id_d + (d/n) P_U} = \sigma \Id_d + (\sqrt{d/n + \sigma^2} - \sigma) P_U$), and we reach that for any $\mu \in [n]$:
\begin{equation}\label{eq:cv_mz}
    \left|
     \frac{|\langle \bx_\mu, \tilde{\bx}(\bz)\rangle|}{\norm{\bx_\mu} \norm{\tilde{\bx}(\bz)}}  
    - \frac{|\bx_\mu^\T P_U \bz|}{\|\bx_\mu\|\sqrt{\bz^\T P_U \bz + n \sigma^2}}\right| \to 0.
\end{equation}
For a fixed $n$ as $d \to \infty$, since $\|\bu_\mu\|/\sqrt{d} \to 1$, and using \cref{eq:decomp_xmu} we have as $d \to \infty$: 
\begin{equation*}
   \left| \frac{\bx_\mu^\T P_U \bz}{\|\bx_\mu\|} - \bz^\T \bomega_\mu \right| \to 0.
\end{equation*}
We defined $\bomega_\mu \coloneqq \bu_\mu / \|\bu_\mu\|$, which are a finite number of vectors that become asymptotically orthonormal in the high-dimensional limit: $\bomega_\mu \cdot \bomega_\nu \to 0$. 
All in all, we reach that as $d \to \infty$ while $n \geq 1$ remains finite:
\begin{equation}\label{app:eq:r_smalln}
    m \sim \EE_{z_1, \cdots, z_n \iid \mcN(0,1)} \left[\frac{\max_{1 \leq \mu \leq n} |z_\mu|}{\sqrt{n \sigma^2 + \sum_{\mu=1}^n z_\mu^2}}\right].
\end{equation}
Taking $n \gg 1$ in \cref{app:eq:r_smalln} (\emph{after} $d \to \infty$), one gets
\begin{align}\label{app:eq:r_n_order1}
    m \simeq \sqrt{\frac{2 \log n}{n(1+\sigma^2)}}
\end{align}
in this limit. In particular the transition from $m$ being close to $1$ to close to $0$ occurs on the scale
$n = \Theta(1)$.

\myskip

\subsection{Convergence}%
\label{app:theory-convergence}

To quantify the convergence of the generative models 
instead, \citet{kadkhodaie2024generalization} computed the average overlap
between two samples generated by the two independently trained generative models
from the same latent variable $\bz$. In our setup, given two independently drawn data sets $\mathcal{D}^{(b)}, b \in \{0, 1\}$ of $n$
samples $\bx^{(b)}_{\mu=1,\ldots,n} \iid \mathsf{P}^\star$,
where~$\mathsf{P}^\star$ is the ground truth distribution of the inputs, and a set of
latent variables $\bz_i \iid \mcN(0, \Id_d)$, we define convergence via the overlap 
\begin{equation}\label{eq_app:def_q}
  q \coloneqq \EE_\bz \frac{|\langle \bx^{(0)}(\bz), \bx^{(1)}(\bz) \rangle|}{\norm{\bx^{(0)}(\bz)} \norm{\bx^{(1)}(\bz)}},
\end{equation}
where $\bx^{(b)}(\bz)$ are two samples drawn from the generative model fitted to
the datasets $\mathcal{D}^{(b)}$, respectively, given the same latent vector
$\bz$.
We sketch below the computation of the limiting behaviour of $q$ as $d \to \infty$ for different regimes of $n$, and $\mathsf{P} = \mcN(0, \Id_d + \beta \bu \bu^\T)$. We consider $\sigma = 0$, and detail in the end how to generalize the computation to $\sigma \geq 0$.

\subsubsection{The finite-$n$ regime} 

In the regime $n = \Theta(1)$ as $d \to \infty$, one can generalize the argument detailed above for the behaviour of the memorization parameter $m$. 
We sketch the argument at a heuristic level below. Recall that 
we showed that, to leading order in this limit $\hSigma \sim (d/n) P_U$, where $P_U$ is an orthogonal projectors on $U$, a uniformly-sampled random $n$-dimensional subspace of $\bbR^d$. 
Applying this to the two independent datasets, we reach that 
\begin{equation*}
    \frac{|\langle \bx^{(0)}(\bz), \bx^{(1)}(\bz) \rangle|}{\norm{\bx^{(0)}(\bz)} \norm{\bx^{(1)}(\bz)}} \sim \frac{|\bz^\T P_U P_V \bz|}{\sqrt{\bz^\T P_U \bz} \sqrt{\bz^\T P_V \bz}}.
\end{equation*}
and where $P_U$, $P_V$ are orthogonal projectors on $U, V$, two independently-sampled random $n$-dimensional subspaces of $\bbR^d$.
If we denote $P_U = UU^T$, where $U \in \bbR^{d \times n}$ 
is a matrix whose $n$ columns $(\bu_1, \cdots, \bu_n)$ are Haar-distributed random orthonormal vectors (and similarly for $P_V$), we reach:
\begin{equation*}
    \frac{|\langle \bx^{(0)}(\bz), \bx^{(1)}(\bz) \rangle|}{\norm{\bx^{(0)}(\bz)} \norm{\bx^{(1)}(\bz)}} \sim \frac{|\bx^\T U^\T V \by|}{\|\bx\| \|\by\|},
\end{equation*}
where $\bx, \by$ are vectors in $\bbR^n$ drawn from a Gaussian distribution with zero mean, and covariance matrix 
\begin{equation*}
    \begin{pmatrix}
        \Id_n & U^\T V \\ 
        V^\T U & \Id_n
    \end{pmatrix}.
\end{equation*}
For $n = \Theta(1)$ as $d \to \infty$, 
the matrix $U^\T V$ has elements 
$(U^\T V)_{\mu \nu} = \bu_\mu \cdot \bv_\nu$.
Thus by the central limit theorem $U^\T V \sim G / \sqrt{d}$, where $G \in \bbR^{n \times n}$ is a matrix with i.i.d.\ standard Gaussian elements. 
All in all we reach that to leading order 
\begin{equation*}
    \frac{|\langle \bx^{(0)}(\bz), \bx^{(1)}(\bz) \rangle|}{\norm{\bx^{(0)}(\bz)} \norm{\bx^{(1)}(\bz)}} \sim \frac{|g|}{\sqrt{d}},
\end{equation*}
where $g \sim \mcN(0,1)$. In other words, in this regime $n = \Theta(1)$ the correlation between the two samples is the same as the correlation between two randomly-sampled vectors in $\bbR^d$.


\subsubsection{Linear regime $n \asymp d$}
\label{app:sec:q_linear_regime}

At linear sample complexity, the generalisation overlap $q$ has a simple expression as well. 
Indeed, in this regime the matrices $\hat{\Sigma}^{(0)}, \hat{\Sigma}^{(1)}$ have well-defined asymptotic spectra, given by a (shifted) Marchenko-Pastur law, as described below.
In particular, by the law of large numbers we have that both the numerators and denominators in \cref{eq_app:def_q} concentrate on deterministic values:
\begin{equation*}
  q = \EE \frac{\tr
      \left[\sqrt{\hat{\Sigma}^{(0)}}\sqrt{\hat{\Sigma}^{(1)}}\right]}{\sqrt{\tr
        \hat{\Sigma}^{(0)} \; \tr \hat{\Sigma}^{(1)}}} + \smallO(1).
\end{equation*}
We denoted $\tr(M) \coloneqq (1/d) \Tr[M]$ for $M$ a $d \times d$ matrix.
Notice that we removed the absolute value in the numerator, since the product of two positive-semidefinite matrices has a non-negative trace.
Further, we have as $d \to \infty$ the strong concentration:
\begin{equation*}
    \tr[\hSigma^{(0)}] = \frac{1}{nd} \sum_{\mu=1}^n \|\bx_\mu\|^2 \to 1.
\end{equation*}
Thus as $d \to \infty$:
\begin{equation}
  \label{app:eq:q_proportional}
  q = \EE \, \tr
      \left[\sqrt{\hat{\Sigma}^{(0)}}\sqrt{\hat{\Sigma}^{(1)}}\right] + \smallO(1)
\end{equation}
Recall the decomposition in \cref{eq:hSigma}: the low-rank perturbations involving $\alpha(\beta)$ do not contribute to the asymptotic spectrum, and one can easily check that, up to a $\smallO(1)$ correction, we can replace \cref{app:eq:q_proportional} by:
\begin{equation}
  \label{eq_app:q_proportional_2}
  q = \EE\, \tr
      \left[\sqrt{\frac{1}{n} \sum_{\mu=1}^n \bu_\mu \bu_\mu^\T}\sqrt{\frac{1}{n} \sum_{\mu=1}^n \bv_\mu \bv_\mu^\T}\right] + \smallO(1),
\end{equation}
where $\bu_\mu, \bv_\mu$ are independent vectors sampled from $\mcN(0, \Id_d)$.
We can exploit this to derive an
exact expression for $q$ as $n, d \to \infty$.  Let $n = \gamma d$, and denote $\mu_\gamma$ the
Marchenko-Pastur distribution with ratio~$\gamma$:
\begin{align}
  \label{app:eq:def_mugamma}
  \mu_\gamma(x) \coloneqq \max(0, 1-\gamma) \delta(x) + \gamma \frac{ \sqrt{(\lambda_+(\gamma) - x)(x - \lambda_-(\gamma))}}{2 \pi x} \indi\{\lambda_-(\gamma) \leq x \leq \lambda_+(\gamma)\},
\end{align}
with 
\begin{align}
  \label{app:eq:mp_edges}
  \lambda_\pm(\gamma) \coloneqq (1 \pm \gamma^{-1/2})^2.
\end{align}
It is known that the limiting spectral distribution of the empirical covariance matrix $(1/n) \sum_{\mu=1}^n \bu_\mu \bu_\mu^\T$ converges to $\mu_\gamma$ as $n, d \to \infty$~\cite{marchenko1967distribution}, and the same holds for the empirical covariance matrix of $(\bv_\mu)_{\mu=1}^n$.
Moreover, notice that both matrices are \emph{rotationally-invariant}: in particular one checks easily that their eigenvectors form a uniformly-sampled (Haar-distributed) orthogonal matrix, and are independent from their eigenvalues~\cite{potters2020first}.
Using the eigendecomposition of the two matrices appearing in \cref{eq_app:q_proportional_2}, we therefore reach:
\begin{align}
  \label{app:eq:q_evalues}
  q  = \frac{1}{d} \EE \sum_{i,j} \sqrt{\lambda_i^{(0)} \lambda_j^{(1)}} \langle \bomega_i^{(0)}, \bomega^{(1)}_j\rangle^2 + \smallO(1).
\end{align}
In \cref{app:eq:q_evalues}, the eigenvectors $\bomega^{(a)}$ (for $a \in \{1, 2\}$) form a Haar-distributed orthogonal matrix, and are independent of the eigenvalues $(\lambda_i^{(a)})$. Since $\EE\left(\bomega_i^{(0)} \cdot \bomega_j^{(1)}\right)^2 = 1/d$ for all $i,j$, we get:
\begin{align*}
    q  =\prod_{b \in \{1, 2\}}\left(\frac{1}{d} \EE \sum_{i=1}^d \sqrt{\lambda_i^{(b)}}\right) + \smallO(1).
\end{align*}
Taking the large-$d$ limit finally gives:
\begin{align}\label{app:eq:q_n_order_d}
    q \simeq \left(\int \mu_\gamma(\rd x) \, \sqrt{x}\right)^2,
\end{align}
which is the expression we plot in \cref{fig:figure1}. 
Notice that for $\gamma \to \infty$, $\mu_\gamma(x) \to \delta(x-1)$, and so $q \to 1$ as $n / d \to \infty$. 
Conversely, 
as $\gamma \to 0$, $\mu_\gamma$ has mass $(1-\gamma)$ at $0$, and a mass $\gamma$ in a region centered around $1/\gamma$, with width $\mcO(1/\sqrt{\gamma})$. This implies that $q \sim \gamma$ in the limit $\gamma \downarrow 0$.
In particular, the transition from an overlap $q$ close to $0$ to an overlap close to $1$ occurs in the scale $n = \Theta(d)$.

\myskip
For the regularised estimator with $\sigma \geq 0$, we can repeat the calculation
above by noting that all the eigenvalues are shifted by $\sigma^2$ -- the
regulariser effectively sets a lower limit of $\sigma^2$ on the smallest
eigenvalue. A quick calculation along the same lines as above then shows that
the convergence order parameter $q$ is given by
\begin{equation}
  q = \frac{1}{1 + \sigma^2}\left(\int \mu_\gamma(\rd x) \sqrt{x + \sigma^2}\right)^2.
\end{equation}

\subsection{Subspace recovery and the BBP transition}%
\label{app:sec:bbp}

The derivations above have a striking consequence: neither the memorization order parameter $m$ nor the convergence overlap $q$ depend, to leading order, on the signal-to-noise ratio $\beta$. More formally, these parameters do not depend on the small-rank changes in the covariance $\Sigma$ of the data, but only on its asymptotic spectrum. 

\myskip
As stated in the main text,
this suggests to introduce a third parameter $Q \in [0,1]$ to measure generalisation, which is the overlap between the two top eigenvectors of $\hat{\Sigma}^{(0)}$ and $\hat{\Sigma}^{(1)}$:
\begin{equation}
    Q \coloneqq |\langle \bv_{\max}(\hat{\Sigma}^{(0)}), \bv_{\max}(\hat{\Sigma}^{(1)}) \rangle |.
\end{equation}
We can obtain $Q$ from the classical literature on the so-called ``Baik Ben-Arous Péché'' (BBP) transition~\cite{baik2005phase}.
We now recall this result, as stated in~\cite{benaych2011eigenvalues}. 
Recall $\hat{\Sigma} = (1/n) \sum_{\mu=1}^n \bx_\mu \bx_\mu^\T$, with $\bx_\mu \iid \mcN(0, \Id_d + \beta \bu \bu^\T)$. 
We define 
\begin{align}
  \label{app:eq:delta_BBP}
  \delta(\beta,\gamma) \coloneqq
  \begin{dcases}
    0 & \textrm{ if } \gamma < \gamma_c(\beta) \coloneqq \beta^{-2}, \\
    \sqrt{\frac{\beta^2 \gamma - 1}{\beta(1 + \beta \gamma)}} & \textrm{ if }
    \gamma \geq \gamma_c(\beta).
  \end{dcases}
\end{align}
Then one has that the top eigenvector of $\hSigma$ satisfies: 
\begin{align}
  \label{app:eq:vmax}
    \lim_{d \to \infty} |\langle \bv_{\max}(\hat{\Sigma}) , \bu \rangle| = \delta(\beta,\gamma),
\end{align}
and moreover, the distribution of $ \bv_{\max}(\hat{\Sigma})$ is invariant by any orthogonal transformation $O$ such that $O\bu = \bu$.
For two independent draws $\hSigma^{(0)}, \hSigma^{(1)}$, we 
thus reach that
\begin{align}\label{app:eq:Q_delta}
\nonumber
    Q &= 
    |\langle \bv_{\max}(\hSigma^{(0)}) , \bu \rangle 
    \langle \bv_{\max}(\hSigma^{(1)}) , \bu \rangle
    + \langle P_{\bu^\perp}\bv_{\max}(\hSigma^{(0)}) , 
    P_{\bu^\perp}\bv_{\max}(\hSigma^{(1)}) \rangle|, \\ 
    &= \delta(\beta,\gamma)^2 + \smallO(1).
\end{align}
Notice in particular that $Q=0$ for $n < n_c(\beta) \sim \gamma_c(\beta) d$, while $Q \to 1$ for $\gamma = n /d \to \infty$. 

\clearpage

\section{From overlaps to statistical distances}

\subsection{Kullback-Leibler divergence and total variation distance are not
affected by a finite-rank perturbation of the covariance}%
\label{app:sec:what-spike}

We first give a brief argument that shows that neither the Kullback-Leibler nor
the total variation distance between two multivariate normal distributions is
affected by the presence of a rank-one perturbation to the covariance (i.e., to
a spike).

For any injective transformation $f : \bbR^d \to \bbR^d$ and any probability
distributions $P, Q$ on $\bbR^d$, we have (denoting $f \# P$ the push-forward
measure, i.e.\ the law of $f(X)$ for $X \sim P$)
\begin{align}\label{eq:pushf_distance}
\begin{dcases}
    \rd_{\rm KL}(P || Q) &= \rd_{\rm KL}(f\#P || f\#Q), \\
    \rd_{\rm TV}(P,Q) &= \rd_{\rm TV}(f\#P, f\#Q).
\end{dcases}
\end{align}
Let $P^\star = \mcN(0, \Sigma^\star)$ and $P = \mcN(0, \hat{\Sigma})$. Notice that
$P^\star = M_{\sqrt{\Sigma^\star}} \# \mcN(0, \Id_d)$, where $M_{\sqrt{\Sigma^\star}}(\bx)
\coloneqq \sqrt{\Sigma^\star} \, \bx$ is just the left-multiplication by
$\sqrt{\Sigma^\star}$ (which is injective since $\Sigma^\star = \Id + \beta \bu \bu^\T$ is
full-rank). Moreover, 
\begin{equation}
  \hat{\Sigma} = \frac{1}{n}\sum_{i=1}^n \bx_i \bx_i^\T = \sqrt{\Sigma^\star} \left(\frac{1}{n}  \sum_{i=1}^n \bz_i \bz_i^\T \right) \sqrt{\Sigma^\star},
\end{equation}
where $\bz_i \iid \mcN(0, \Id_d)$. Denoting $S \coloneqq (1/n) \sum_{i=1}^n
\bz_i \bz_i^\T$, we have $P = M_{\sqrt{\Sigma^\star}} \# \mcN(0, S)$. Applying
eq.~\eqref{eq:pushf_distance} we get 
\begin{equation}
    \rd_{\rm TV}(P, P^\star) = \rd_{\rm TV}(\mcN(0, \Id_d), \mcN(0, S)),
\end{equation}
and similarly for the KL divergence. Neither quantity hence depends on the
presence of the spike, which leads us to consider the maximum-sliced distance to
elucidate the role of $Q$ in \cref{app:sec:MS-Q}.

\subsection{Bounds on $\DKL$}%
\label{app:sec:q_DKL_proof}

We start with an analysis of the (rescaled) Kullback-Leibler $\DKL$ divergence
between ground-truth and estimated distributions. From our previous argument,
the Kullback-Leibler divergence is not affected by the presence of a spike in
the covariance, so to leading order we can consider the case $\beta = 0$. Since
there will always be a slight mismatch between the ground-truth and the
estimated distribution even when $n$ grows large due to the regularisation, we
consider an ``excess'' $\DKL$ between the learnt and the ``regularised''
ground-truth covariance $\Sigma^* + \sigma^2 \Id_d$:
\begin{align*}
  \mathcal{D} & \coloneqq \frac{1}{d} \EE \DKL \left[\mcN(0, \hat{\Sigma}) | \mcN\left(0, (1 + \sigma^2)\Id_d\right) \right] \\
  & =  \frac{1}{2d} \EE \left((1 + \sigma^2)^{-1} \tr \hat{\Sigma} - d + \log \det (1 + \sigma^2)\Id_d - \log \det \hat{\Sigma} \right)\\
  &=-\frac{1}{2} \mathbb{E}_\gamma \log \nu \label{eq:dkl_nu}
\end{align*}
where we have used the fact that $\EE_\gamma \lambda = 1$ and we have introduced
the rescaled eigenvalue
\begin{equation}
  \label{eq:nu}
  \nu \coloneqq \frac{\lambda + \sigma^2}{1 + \sigma^2},
\end{equation}
to account for the shift of all the eigenvalues $\lambda$ of $\hat{\Sigma}$ by
$\sigma^2$. The normalisation ensures that~$\EE_\gamma \nu = 1$. Note that one can then derive 
an exact formula for $\mcD$ that exploits the Stieltjes transform of the
Marchenko-Pastur distribution, see eq.~(4.40) of \citet{potters2020first}.
Instead of writing such an exact expression, here we are interested in the relation
between $\mathcal{D}$ and the convergence overlap $q$, which we establish via
the following explicit upper and lower bounds.

\subsubsection{Lower bound}

Our lower bound is a direct consequence of Jensen's inequality. Noting that we
can rewrite the convergence overlap in terms of $\nu$ as $q = {\left(\EE_\gamma
\sqrt \nu\right)}^2$, we have that $\EE_\gamma \log \nu = 2 \EE_\gamma \log
\sqrt \nu \le 2 \log \EE \sqrt \nu = \log q$, so 
\begin{equation}\label{eq:d_lb}
  -\EE_\gamma \log \nu \ge - \log q \Rightarrow 
  \mathcal{D} \ge -\frac{1}{2}\log q.
\end{equation}

\subsubsection{Upper bound}
\label{app:sec:dkl_bound_all_gamma}

We now derive the following bound on the Kullback-Leibler divergence which is true for any $\gamma > 0$:
\begin{equation}\label{eq:ub_D_alt}
    \mcD \leq 2 \left[1 + \log(1 + \sigma^{-2})\right] (1-q).
\end{equation}

\begin{proof}[Proof of \cref{eq:ub_D_alt}]
Let $r = \sqrt{\nu}$.
Notice that (for any $\gamma > 0$) we have $r_{\min} \leq r \leq r_{\max}$ with $r_{\min}^2 = \sigma^2 / (1 + \sigma^2) \in (0,1)$ and $r_{\max}^2 = (\lambda_+ + \sigma^2)/(1+\sigma^2) > 1$.
Moreover, $q = \EE[r]^2$ and $\EE[r^2] = 1$.
We use the following elementary bounds proven below.
\begin{lemma}\label{lemma:ub_log}
\begin{itemize}
    \item[$(i)$] For any $a > 0$, and any $x \geq a$: 
    \begin{equation*}
        -\log x \leq \left[\frac{a-1 - \log a}{(a-1)^2}\right] (x-1)^2 - (x-1).
    \end{equation*}
    \item[$(ii)$] For any $x \in (0,1]$: 
    \begin{equation*}
        \frac{x^2 - 1 - 2\log x}{(x-1)^2} \leq 2[1-\log x].
    \end{equation*}
\end{itemize}
\end{lemma}
\noindent
Applying Lemma~\ref{lemma:ub_log}-$(i)$ to $\mcD = - \EE \log r$, we get 
\begin{align*}
    \mcD &\leq \left[\frac{2(r_{\min}-1 - \log r_{\min})}{(r_{\min}-1)^2} + 1\right] (1-\sqrt{q}) \leq \left[\frac{r_{\min}^2-1 - 2\log r_{\min}}{(r_{\min}-1)^2}\right] (1-q).
\end{align*}
Applying then Lemma~\ref{lemma:ub_log}-$(ii)$, we get eq.~\eqref{eq:ub_D_alt}.
\end{proof}
\begin{proof}[Proof of Lemma~\ref{lemma:ub_log} --]
Start with $(i)$ and 
fix $a > 0$. Let $A = (a - 1 - \log a)/(a - 1)^2$, with $A = 1/2$ if $a=1$ by continuity. We want to show that $f(x) = A(x-1)^2 - (x-1) + \log x \geq 0$ for all $x \geq a$. 
Observe that $f(1) = 0$ and $f(a) = 0$. The derivative is:
\begin{equation*}
f'(x) = 2A(x-1) - 1 + \frac{1}{x} = \frac{(x-1)(2Ax - 1)}{x}.
\end{equation*}
The roots of $f'(x)$ are $\{1, 1/(2A)\}$. Let $z \coloneqq (2A)^{-1}$. Since $f(1)=f(a)=0$, we must have $\min(1, a) < z < \max(1, a)$.
\begin{itemize}[leftmargin=*]
    \item \textbf{Case $a < 1$:} We have $a < z < 1$. 
    Notice that $f'(a) > 0$, so that $f$ is strictly increasing on $(a, z)$ and strictly decreasing on $(z, 1)$, and thus $f(x) \geq 0$ on $[a, 1]$. 
    Moreover, for any $x \geq 1$, $f'(x) \geq 0$, and thus $f(x) \geq 0$ for all $x \geq a$.
    \item \textbf{Case $a \geq 1$:} We have $1 \leq z \leq a$. 
    Thus for any $x \geq a$, we have $f'(x) \geq 0$, and thus $f(x) \geq 0$ for all $x \geq a$.
    This completes the proof of (i).
\end{itemize}
Let us now prove $(ii)$.
We aim to show that $g(x) \geq 0$ for $x \in (0, 1]$, where:
\begin{align*}
g(x) = 2(1 - \log x)(x-1)^2 - (x^2 - 1 - 2\log x) = x^2 - 4x + 3 - 2x(x-2)\log x.
\end{align*}
Notice that $g(1) = 0$.
Moreover an explicit computation yields $g'(x) = 4(1-x)\log x \leq 0$, 
and thus $g(x) \geq 0$ for $x \in (0, 1]$.
\end{proof}

\paragraph{Remark on the regime $\gamma \geq 1$} We note that one can also refine this bound as follows: using the fact that $r_{\min}$ is bigger for $\gamma \geq 1$, we immediately get
\begin{equation}\label{eq:ub_D_alt_2}
    \mcD \leq 2 \left[1 - \log\left(\frac{\lambda_- \indi\{\gamma \geq 1\} +\sigma^2}{1+\sigma^2}\right)\right] (1-q).
\end{equation}
In particular, the multiplicative term in front of $(1-q)$ goes to $2$ for large $\gamma$ and any finite $\sigma \geq 0$ (even for $\sigma = 0$).
Notice that this term diverges for $\gamma = 1$ and $\sigma \downarrow 0$, while $\mcD$ itself does not diverge: this is due to the sub-optimality of our bound, which only exploits the first two moments and the support of the distribution. While one could surely refine the bounds further, eq.~\eqref{eq:ub_D_alt_2} and eq.~\eqref{eq:d_lb} already show that, 
as $\gamma \to \infty$ (and thus $q \to 1$), $\mcD$ decays as $(1-q)$ up to constant factors.

\subsection{Maximum-sliced distance and its decay with $Q$}%
\label{app:sec:MS-Q}

\subsubsection{Definition of maximum-sliced distance}

In this section, we focus on the maximum-sliced distance $\MS$ between
$\mathsf{P}^\star$ and the estimated density $\hat{\mathsf{P}} = \mcN(0,
\hat{\Sigma})$. To define $\MS$, we first define the integral probability metric (IPM)
$d(P, Q)$ between two probability measures $P$ and $Q$ on $\reals$,
\begin{align}\label{eq:d_IPM}
  d(P,Q) \coloneqq \sup_{f\, : \, \|f''\|_\infty \le 1} \left| \mathbb{E}_P[f] - \mathbb{E}_Q[f]\right|, 
\end{align}
where $\|f\|_\infty$ is the uniform norm of $f$. The definition of $d(P, Q)$ has
a similar form to total variation distance, except that the $\sup$ is not over
events (i.e.\ over indicator functions), but over smooth functions. One can show
that $d$ is finite if $P, Q$ have the same mean (which one can take equal to $0$
without loss of generality), and that $d$ forms a distance over the set of
zero-mean probability distributions. The definition \cref{eq:d_IPM} has the
advantage over total variation that it yields a very simple closed-form
expression for normal distributions. For $P = \mcN(0, \sigma^2)$ and $Q =
\mcN(0, \rho^2)$, it is easy to show that:
\begin{equation}
  \label{eq:d_gaussian_1d}
  d(P, Q) = \frac{1}{2} |\sigma^2 - \rho^2|.
\end{equation}

\begin{proof}[Proof of \cref{eq:d_gaussian_1d}]
    Notice that $d(P, Q) \geq |\rho^2 - \sigma^2|/2$ trivially by taking $f(x) = x^2/2$, we thus focus on the other bound.
    Assume without loss of generality that $\sigma^2 \ge \rho^2$. Let $Y \sim \mcN(0, \rho^2)$, 
    and $Z \sim \mathcal{N}(0, \sigma^2 - \rho^2)$ Then $X \coloneqq Y + Z \sim \mcN(0, \sigma^2)$.
    Let $f$ be such that $\|f''\|_\infty \le 1$. Applying a second-order Taylor expansion we obtain:
    \begin{equation*}
        f(Y+Z) = f(Y) + Z f'(Y) + \frac{1}{2} Z^2 f''(\xi)
    \end{equation*}
    where $\xi$ is a random variable lying strictly between $Y$ and $Y+Z$.
    Taking the expectation over both sides, and using the known moments of $Z$ yields:
    \begin{equation*}
        \mathbb{E}[f(X)] = \mathbb{E}[f(Y+Z)] = \mathbb{E}[f(Y)] + \frac{\sigma^2 - \rho^2}{2} \mathbb{E}[f''(\xi)]
    \end{equation*}
    Using finally that $\|f''\|_\infty \leq 1$ concludes the proof.
\end{proof}
Given probability measures $P$ and $Q$ on $\reals^d$, the \textbf{maximum-sliced
(MS)} distance is defined as
\begin{align}
  d_{\mathrm{MS}}(P,Q) \coloneqq \sup_{\bv \, : \, \|\bv\| \le 1} d(\bv^\top P, \bv^\top Q) 
\end{align}
where $\bv^\top P$ denotes the one-dimensional distribution corresponding to
the projection of $P$ into the direction of $\bv$. It can be verified that
the MS distance is a metric~\cite{kolouri2019generalized} and that convergence
with respect to $d_{\mathrm{MS}}$ implies convergence in distribution as well as
convergence of second moments. For two multivariate normal distributions $P =
\mcN(0, \Sigma^\star)$ and $Q = \mcN(0, \hat{\Sigma})$, we have that
\begin{equation}
  \bv^\top \mathsf{P}^\star = \mcN(0, \bv^\top \Sigma^\star \bv) \quad \text{and} \quad  \bv^\top \hat{\mathsf{P}} = \mcN(0, \bv^\top \hat{\Sigma} \bv).
\end{equation}
Applying the result from \cref{eq:d_gaussian_1d} to these two distributions, we
find that the maximum-sliced distance between the two multivariate Gaussians is 
\begin{equation}\label{eq:MS_gen}
  \MS(\mathsf{P}^\star, \hat{\mathsf{P}}) = \sup_{\bv \, : \, \|\bv\| \le 1} \frac{1}{2}  |\bv^\top (\Sigma^\star - \hat{\Sigma})\bv | = \frac{1}{2} {\|\Sigma^\star - \hat{\Sigma}\|}_\op,
\end{equation}
where recall that ${\|A\|}_\op$ is the operator (spectral) norm of a matrix $A$, i.e.\ its largest
singular value.

\subsubsection{Exact computation of maximum-sliced distance and decay with $Q$}

We now compute the maximum-sliced distance between the empirical distribution
$\hat{\mathsf{P}}$ and the ground-truth distribution, where we again add a
``regularisation'', as in our analysis of the excess Kullback-Leibler
divergence. From \cref{eq:MS_gen}, we see that the value of $\sigma$ does not
influence the value of this ``excess'' MS distance. We thus have to compute the
spectral norm
\begin{equation}
  \Lambda = {\|\Delta\|}_\op, \quad \Delta \coloneqq \Sigma^\star - \hat{\Sigma} 
  = {\left( \Sigma^\star\right)}^{1/2} (\Id_d - W) {\left( \Sigma^\star\right)}^{1/2}
\end{equation}
where $\Sigma^\star = \Id+\beta \bu\bu^\T$, and $W$ is a Wishart matrix with aspect ratio $\gamma$.

\paragraph{The bulk contribution} To understand the behaviour of the bulk
eigenvalues of $\Delta$, we first note that ${\left( \Sigma^\star\right)}^{1/2}$
is a finite-rank perturbation of $\Id_d$, so the bulk law of $\Delta$ follows
the Marchenko-Pastur distribution under the variable shift $t \to 1 - t$. With
$\lambda_\pm(\gamma) \coloneqq (1 \pm \gamma^{-1/2})^2$ (see above), we therefore
have
\begin{equation}\label{eq:Lambda_bulk}
  \Lambda_{\mathrm{bulk}} = \max(|1 - \lambda_+|, |1 - \lambda_-|) = \frac{1}{\gamma} + \frac{2}{\sqrt \gamma}.
\end{equation}

\paragraph{The spike contribution} As we increase the sample complexity, an
eigenvalue will detach from the bulk in a phase transition analogous to the BBP
transition of the empirical covariance matrix $\hat{S}$, \cref{app:sec:bbp}. To
understand how this spike affects the spectral norm, we need to compute the
critical sample complexity at which the outlier appears, and its value. To this
end, we first compute the Stieltjes transform for $\Delta$. A standard
calculation (see for example sec. 3.3 of \citet{bandeira2025topics}) leads to
the following relation between the signal-to-noise ratio $\beta$ and the
Marchenko-Pastur density, 
\begin{equation}
  \frac{1}{\beta} = \int \mu^{\mathrm{MP}}_\gamma(\rd t) \frac{1 - t}{\lambda - (1 - t)}
\end{equation}
We can evaluate the integral by rewriting the integrand as
\begin{equation}
  \frac{1 - t}{\lambda - (1 - t)} = \frac{\lambda}{\lambda - (1 - t)} - 1 = \frac{\lambda}{t- (1 - \lambda)} - 1,
\end{equation}
which allows us to rewrite the integral in terms of the Stieltjes transform
of the Marchenko-Pastur distribution,
\begin{equation}
  g^{\mathrm{MP}}(z) = \int \mu^{\mathrm{MP}}_\gamma(\rd t) \frac{1}{t - z},
\end{equation}
as
\begin{equation}
  \frac{1}{\beta} = \lambda g^{\mathrm{MP}}(1 - \lambda) -1.
\end{equation}
Using the standard result for the Stieltjes of the Marchenko-Pastur (eq.~3.35 of
\citet{bandeira2025topics}), we have finally that any outlier $\lambda$ must satisfy:
\begin{equation}
  \label{app:eq:Delta_stieltjes}
\frac{1}{\beta }=\frac{\lambda  \left(\gamma\lambda - 1 + \gamma\sqrt{\left(1 -\lambda -\lambda
   _-\right) \left(1 -\lambda -\lambda _+\right)}\right)}{2 (1-\lambda)}-1.
\end{equation}
The spectral norm is given by the eigenvalue with the largest absolute value. In
a vanilla spiked Wishart matrix, the eigenvalue detaches from the bulk at the
right edge. 
We study first the left edge of the spectrum. Setting $\lambda = -2 \gamma^{-1/2} -
\gamma^{-1}$ in \cref{app:eq:Delta_stieltjes}, we find that the critical
$\beta_\Delta$ at which the transition occurs is given by
\begin{equation}
  \beta_\Delta = 1 + \gamma^{-\nicefrac{1}{2}}.
\end{equation}
The threshold $\beta_\Delta$ controls the separation of an eigenvalue from the
left edge of the bulk of $\Delta$ and should not be confused with the BBP
threshold $\beta_c = \gamma^{-1/2}$. Note that for $\beta < 1$, no eigenvalue
detaches for any value of $\gamma$. The (absolute) value of the corresponding
eigenvalue can be obtained by inverting \cref{app:eq:Delta_stieltjes} above
$\beta_\Delta$,
\begin{equation}\label{eq:Lambda_spiked}
  \Lambda_{\mathrm{spike}} = \frac{1 + \beta}{2 \beta \gamma} \left(1 + \sqrt{1 + 4 \beta \gamma}\right) \quad \text{for} \quad \beta > 1, \gamma > {(\beta - 1)}^{-2}.
\end{equation}
One can also show from \cref{app:eq:Delta_stieltjes} that a possible eigenvalue detaching from the right edge of the bulk is always smaller (in absolute value) than the maximum of \cref{eq:Lambda_bulk} and \cref{eq:Lambda_spiked}.

\paragraph{Final result and asymptotic expansion} Combining the above results,
we find that the maximum-sliced distance between ground-truth and learnt
distribution is given by
\begin{equation}
  \label{app:eq:dms-exact}
    \MS(\mathsf{P}^\star, \hat{\mathsf{P}}) = \frac{1}{2} {\|\Sigma^\star - \hat{\Sigma}\|}_\op = \frac{1}{2}  \begin{dcases}
    \Lambda_{\mathrm{bulk}} & \textrm{ if } \beta < 1 + \gamma^{-\nicefrac{1}{2}} \\
    \Lambda_{\mathrm{spike}} & \textrm{ if } \beta \ge 1 + \gamma^{-\nicefrac{1}{2}} \\
  \end{dcases}
\end{equation}
Once the spectral norm is given by $\Lambda_{\mathrm{spike}}$, i.e.\ for
$\beta \ge 1 + \gamma^{-\nicefrac{1}{2}}$, we can rewrite~$d_{\mathrm{MS}}$ in
terms of $Q$ using the relation between $\gamma$ and $Q$ from
\cref{eq:Q_delta}. It can be verified that $\MS\to 0$  as $Q \to 1$. An
asymptotic expansion around 0 for $Q = 1 - \epsilon$ with small $\epsilon$ then
yields the asymptotic decay of $\MS$,
\begin{equation}
  \label{app:eq:dms-asymptotic}
  \MS(\mathsf{P}^\star, \hat{\mathsf{P}}) = \frac{1}{2} \sqrt{\beta (1 + \beta)} \sqrt{1 - Q} + \mathcal{O}(1 - Q) \quad \mathrm{as}\quad Q\to1.
\end{equation}

\section{Details on numerical experiments with real images}%
\label{app:sec:experiments}

\subsection{Datasets}%
\label{app:sec:datasets}

We performed our experiments with real images on two data sets: square patches
of width $w\in\{32, 48, 64\}$ sampled randomly from ImageNet images, and the
CelebA database of celebrity images~\citep{liu2015faceattributes}, downsampled
to width $w\in\{32, 48, 64\}$. We converted images from both datasets to
black-and-white  by taking the weighted sum of the colour channels using the
standard formula that takes the relative variance of the channels into account.
Beyond this, we applied the following preprocessing steps:

\begin{description}
\item[ImageNet] To sample the patches, we first selected $n$ images from the
training set of ImageNet at random. In each image, we chose a patch at random,
and after conversion to black-and-white, we centered the images by removing the
spatial mean. Patches with negligible contrast (standard deviation below
$10^{-1}$) were discarded, as they correspond to essentially a constant colour
across the whole patch; this typically removed between $0.1\%$ and $0.5\%$ of all patches.
\item[CelebA] We first center-cropped the images down to squares of width 178
pixels, and then resized the images to the final width $w$. Images were
converted to grayscale and normalised to the range $[-1, 1]$.
\end{description}

\begin{figure*}[t!]
  \centering
  \includegraphics[width=.4\linewidth]{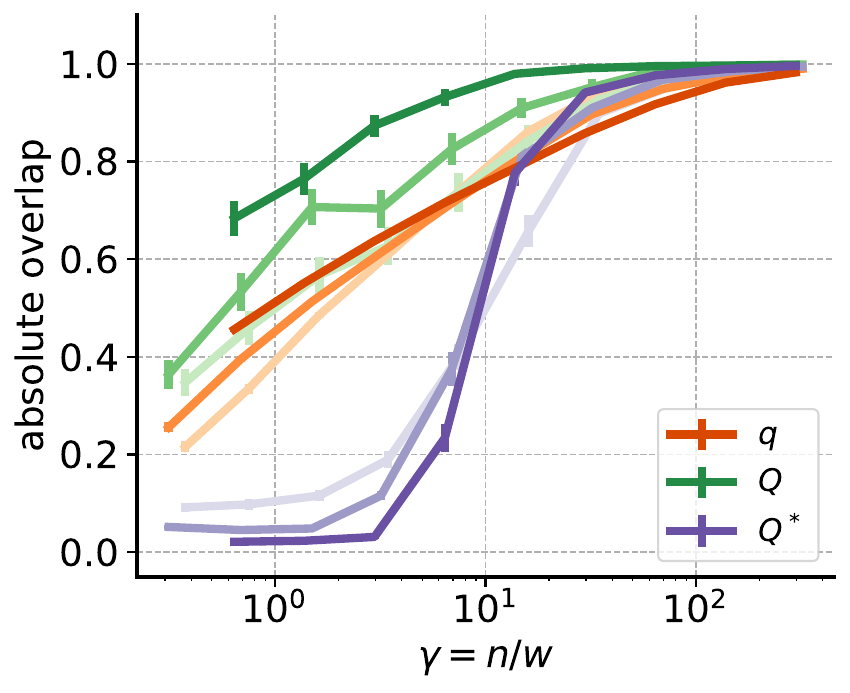}
  \caption{\label{fig:imagenet} \textbf{A rotated subspace overlap identifies a
  phase transition on ImageNet, too} We reproduce the results obtained on CelebA
  in \cref{fig:powerlaw} using random patches from ImageNet as data. We plot the 
  convergence overlap $q$ and subspace overlap $Q$, \cref{eq:q,eq:Q}, for a
  linear model fitted to random patches from ImageNet as a function of sample
  complexity $\gamma = n / w$, where $w \in \{16, 32, 64\}$ is the width of the
  patches (increasingly dark lines). We also plot the rotated subspace overlap
  $Q^\star$ that factors out the power-law decay of the spectrum, \cref{eq:Qstar}.}
\end{figure*}

\subsection{\Cref{fig:powerlaw,fig:imagenet}}

We computed the three order parameters $q, Q$ and $Q^\star$ for the patches at each
width as follows:
\begin{description}
\item[Computing $q$, $Q$] We estimated the empirical covariance from a subset of
$n$ image patches, and then computed $q$ and $Q$ as in the Gaussian case
directly from the definitions \cref{eq:m,eq:q}.
\item[Computing $Q^\star$] We first took the 2D Fourier transform of the image
patches. We then flatten the (complex) coefficients into a vector, normalise by
the coefficient for each mode, and compute the leading eigenvectors of the
correlation matrix for both subsets. The absolute value of their overlap yields
$Q^\star$.
\end{description}

\subsection{\Cref{fig:image_correlations}: Spectra of correlation matrices of
real images}

For each width $w\in\{32, 48, 64\}$, we kept $n = 4w$ images / patches of CelebA
and ImageNet, respectively, and computed their
two-dimensional Fourier transform. The resulting Fourier coefficients were
reshaped into vectors and centered across samples by subtracting the empirical
mean over the data set at each frequency. Each frequency component was then
normalized to unit empirical variance, yielding a standardized, complex data
matrix $Z$. We then computed the eigenvalues of the empirical correlation matrix
of $Z$. In \cref{fig:image_correlations}, the first three plots show the
histogram density of the rescaled eigenvalues $\lambda / w$ on a log-scale. This
representation ensures that the bulk shape is comparable across widths; note
that the four spikes that appear at all three widths correspond to a single
eigenvalue, although the magnitude of the spikes decreases as $w$ increases. The
right plot shows the ten largest eigenvalues (in descending order), also
rescaled by $w$, as a function of their rank.

\subsection{\Cref{fig:powerlaw_snr0}: validating $Q^\star$ in the absence of
correlations in $F$-space}%
\label{app:sec:powerlaw_snr0}

In the synthetic power-law model, we can easily destroy the correlations in
$F$-space by setting $\beta = 0$. In images, we can replicate this effect by
multiplying each Fourier coefficient of each image with a random phase as
follows. Working for simplicity with one-dimensional inputs, we write $X_k = A_k
e^{i \phi_k}$ for the~$k$th Fourier coefficient of a given input. The
correlation matrix is then
\begin{equation}
  \rho_{k\ell} = \frac{\Sigma_{k\ell}}{\sqrt{\EE {|X_k|}^2 \EE {|X_\ell|}^2}}
\end{equation}
where
\begin{equation}
  \Sigma_{k\ell} = \EE X_k \overline{X_\ell} = \EE \left[ A_k A_l e^{i(\phi_k - \phi_\ell)}\right]
\end{equation}
and the expectation is taken over the inputs in a data set. Now introducing a
random phase shift means that for every input and every $k$, we instead consider
\begin{equation}
  Y_k = e^{i \theta_k}, \qquad \theta_k \iid \mathrm{unif}[-\pi, \pi].
\end{equation}
Since the phases are independent across images and components, it is immediate
that the correlations $\rho{k \ell} = 0$ for $k \neq \ell$, while the power-law
decay of the amplitudes $\Sigma_{kk} = {|A_k|}^2$ is unchanged, mirroring the
situation when $\beta = 0$ in the power-law model.

\begin{figure*}[t!]
  \centering
  \includegraphics[width=.7\linewidth]{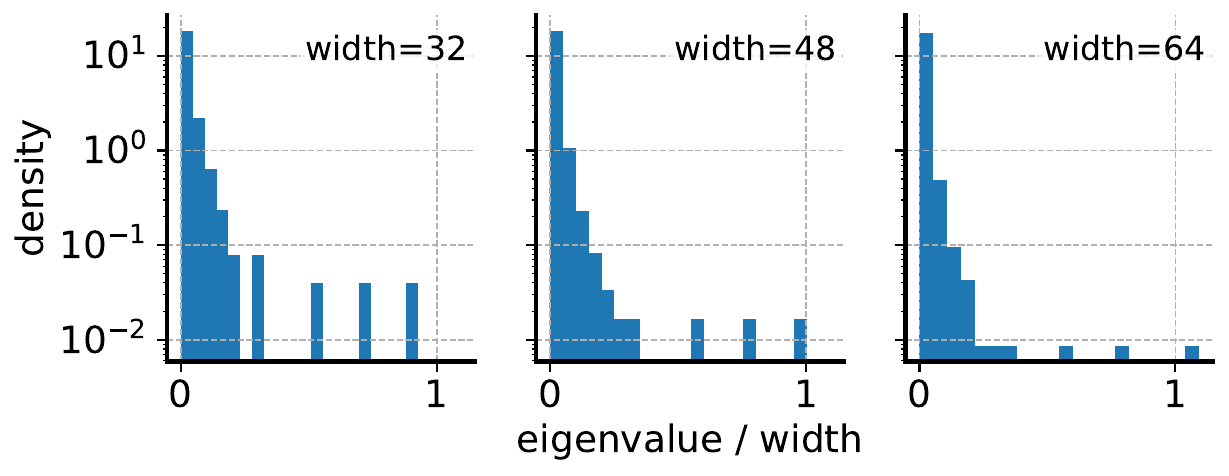}%
  \includegraphics[width=.28\linewidth]{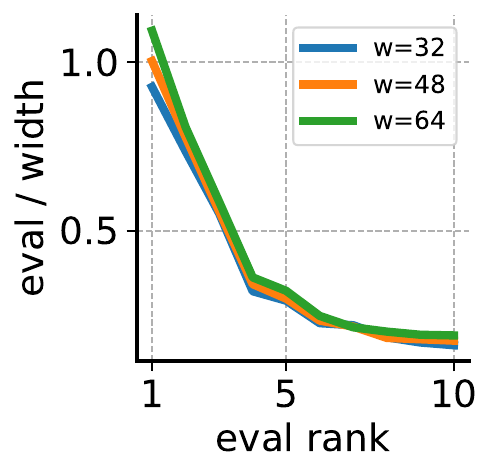}\\
  \includegraphics[width=.7\linewidth]{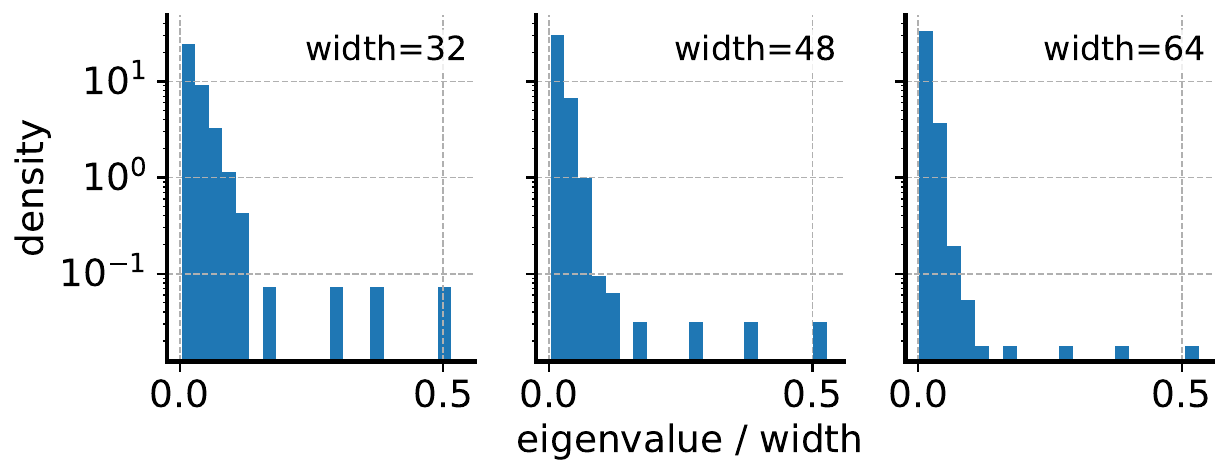}%
  \includegraphics[width=.28\linewidth]{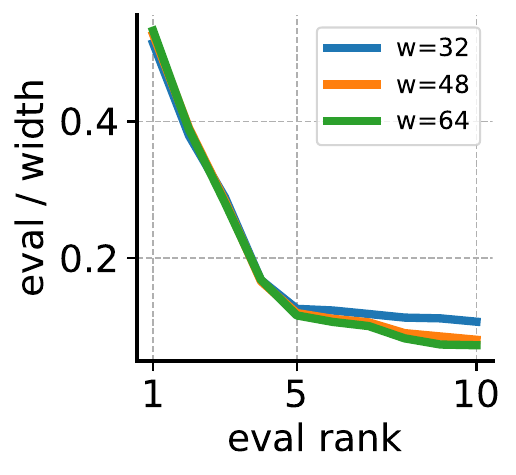}
  \caption{\label{fig:image_correlations} \textbf{The correlation matrices of
  CelebA and ImageNet in Fourier space is low-rank.} We show results for CelebA
  and ImageNet in the top and bottom rows, respectively. \emph{Left}: The three
  histograms show the eigenvalues of the correlation matrix in Fourier space of
  patches of CelebA images and ImageNet patches obtained using the procedure
  described in \cref{app:sec:datasets} for patch widths of $w=32, 48, 64$
  pixels. All spectra are computed from $n=4 w^2$ samples. \emph{Right}: The ten
  leading eigenvalues of the three spectra show a very similar trend when
  rescaled by width.}
\end{figure*}

\begin{figure*}[t!]
  \centering
  \includegraphics[width=.5\linewidth]{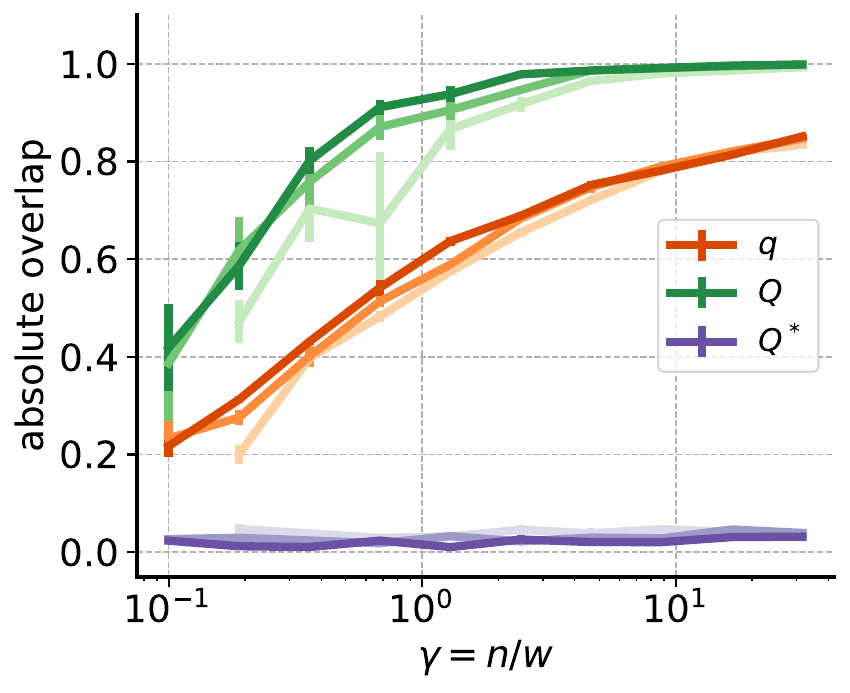}%
  \includegraphics[width=.5\linewidth]{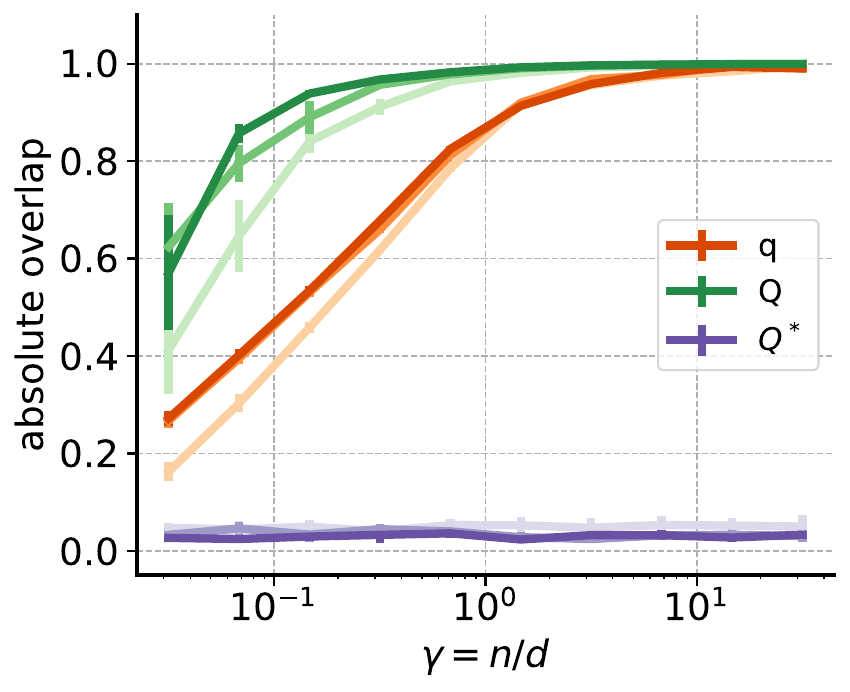}%
  \caption{\label{fig:powerlaw_snr0} \textbf{The rotated subspace overlap $Q^\star$
  remains close to zero if the correlation matrix lacks low-rank structure.}
  \emph{Left:} Same as \cref{fig:powerlaw}(b): we compute convergence $q$,
  subspace overlap $Q$ and subspace overlap $Q^\star$ for CelebA patches, except
  that we now randomly multiplied each Fourier coefficient of each image with a
  random phase shift, see \cref{app:sec:powerlaw_snr0}.  \emph{Right:} Same
  experiment as in \cref{fig:powerlaw}(c) with the synthetic power-law model of
  \cref{eq:powerlaw-model}, except that the correlation matrix does not have
  low-rank structure, i.e.\ $\beta=0$.}
\end{figure*}

\section{Diffusion experiment}%
\label{app:sec:diffusion}

We reproduced the experiment of \citet{kadkhodaie2024generalization} by training
a set of diffusion models independently on disjoint subsets of the CelebA data
set, preprocessed as in \cref{app:sec:datasets}.

\paragraph{Denoiser architecture} We trained a standard encoder-decoder U-Net to
predict the added noise in images. The encoder consists of an initial
double-convolution followed by four downsampling stages with GroupNorm and ReLU
activations; channel width doubles after each downsampling step. The decoder
mirrors this structure with transposed convolutions for upsampling, skip
connections from the corresponding encoder resolutions, concatenation, and
another double-convolution block to fuse features. Diffusion time is injected
throughout the network using a fixed sinusoidal embedding table.

\paragraph{Denoising diffusion probabilistic model (DDPM)} We used the UNet as
the denoiser in a standard DDPM model à la \citet{ho2020denoising} with the
variance-preserving formulation: a linear $\beta_t$ schedule defines $\alpha_t =
1- \beta_t$ and the cumulative products $\bar{\alpha}_t$, which are used in the
forward noising process $x_t = \sqrt{\bar{\alpha}_t} x_0 + \sqrt{1 -
\bar{\alpha}_t} \epsilon$, where $x_0$ is an image and $\epsilon$ is a noise
vector sampled from the standard multivariate normal distribution.

\paragraph{Training} We trained all models by minimising the mean-squared error
in denoising images corrupted by i.i.d.\ Gaussian noise with identity variance.
We trained each model for about 200k steps at mini-batch size 512, which
corresponds to 1000 epochs on the largest data set we considered.

\paragraph{Key differences to \citet{kadkhodaie2024generalization}} We used a
standard U-Net, while they used a ``bias-free'' denoiser and did not provide the
noise level to the denoiser. We sampled using a standard DDPM sampling
procedure, while they sampled using the algorithm of \citet{kadkhodaie2020solving}.

\begin{figure*}[t!]
  \centering
  \includegraphics[width=\linewidth]{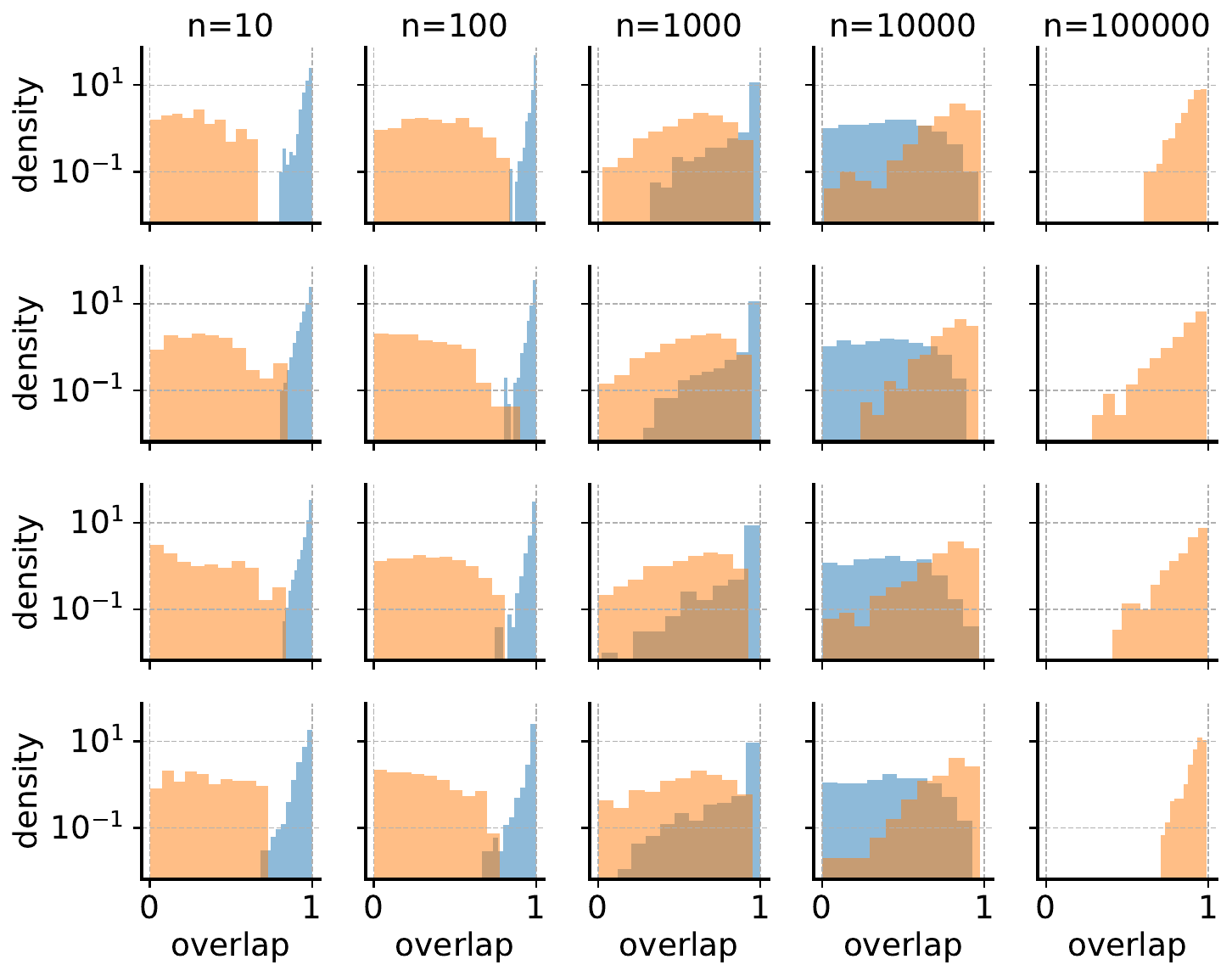}
  \caption{\label{fig:diffusion_m_and_q} \textbf{The transition from
   memorisation to generalisation in denoising diffusion probabilistic models.}
   We reproduce the results of \citet{kadkhodaie2024generalization}, and in
   particular their Figure 2, in our setup described in
   \cref{app:sec:diffusion}. We plot the histograms of the memorization overlap $m$ (blue) and convergence overlap $q$ (orange)
   as a function of sample complexity (columns), for four different repetitions
   of the experiment (rows).}
\end{figure*}